\documentclass{article}

\PassOptionsToPackage{numbers, compress}{natbib}

\usepackage[preprint]{neurips_2026}

\usepackage{mathtools}     
\usepackage{amssymb}

\usepackage{amsthm}
\usepackage{graphicx}
\usepackage{booktabs}
\usepackage{multirow}
\usepackage{colortbl}
\usepackage{array}
\usepackage{makecell}
\usepackage{bbm}
\usepackage{xcolor}
\usepackage{subcaption}
\usepackage{algorithm}
\usepackage{algorithmic}
\usepackage{tikz}
\usetikzlibrary{arrows.meta,positioning,calc,shapes.geometric,fit,backgrounds}
\usepackage{bm}
\usepackage{xspace}
\usepackage[accsupp]{axessibility}
\usepackage[colorlinks=true,allcolors=blue!60!black]{hyperref}
\hypersetup{
  colorlinks=true,
  linkcolor=[rgb]{0.60,0.00,0.00},
  citecolor=[rgb]{0.00,0.38,0.00},
  urlcolor=[rgb]{0.00,0.20,0.60}
}
\usepackage{cleveref}
\usepackage{orcidlink}
\usepackage{pgfplots}
\pgfplotsset{compat=1.18}
\usepackage{textcomp}

\providecommand{\eg}{\textit{e.g.}\@\xspace}

\providecommand{\cf}{\textit{cf.}\@\xspace}

\providecommand{\etal}{\textit{et al.}\@\xspace}

\usepackage{amsmath,amssymb,amsthm}

\newtheorem{theorem}{Theorem}[section]
\newtheorem{proposition}[theorem]{Proposition}
\newtheorem{lemma}[theorem]{Lemma}
\newtheorem{corollary}[theorem]{Corollary}
\theoremstyle{definition}
\newtheorem{definition}[theorem]{Definition}
\newtheorem{assumption}[theorem]{Assumption}
\theoremstyle{remark}
\newtheorem{remark}[theorem]{Remark}

\crefname{assumption}{Assumption}{Assumptions}
\Crefname{assumption}{Assumption}{Assumptions}

\DeclareMathOperator{\CKA}{CKA}
\DeclareMathOperator{\HSIC}{HSIC}
\DeclareMathOperator{\SNR}{SNR}

\newcommand{\R}{\mathbb{R}}
\newcommand{\E}{\mathbb{E}}

\newcommand{\Poof}{\mathbf{P}_{\mathrm{OOF}}}
\newcommand{\Peff}{\mathbf{P}_{\mathrm{eff}}}
\newcommand{\Xmeta}{\mathbf{X}_{\mathrm{meta}}}
\newcommand{\Chat}{\widehat{\mathbf{C}}}
\newcommand{\Sighat}{\widehat{\boldsymbol{\Sigma}}}
\newcommand{\wlam}{\hat{\mathbf{w}}_{\lambda}}
\newcommand{\tr}{\mathrm{tr}}
\newcommand{\diag}{\mathrm{diag}}
\newcommand{\kap}{\kappa}
\newcommand{\lmax}{\lambda_{\max}}
\newcommand{\lmin}{\lambda_{\min}^{+}}
\newcommand{\norm}[1]{\left\|#1\right\|}

\newcommand{\inner}[2]{\langle #1,\, #2\rangle}
\newcommand{\Rademacher}{\mathfrak{R}}

\newcommand{\KL}{\mathrm{KL}}

\newcommand{\ECE}{\mathrm{ECE}}

\newcommand{\method}{\textsc{CoRe-Stack}\xspace}
\newcommand{\methodpp}{\textsc{CoRe-Stack\textsuperscript{+}}\xspace}

\usepackage{pifont}

\definecolor{bestrow}{RGB}{236,252,243}
\newcolumntype{L}[1]{>{\raggedright\arraybackslash}p{#1}}
\newcolumntype{C}[1]{>{\centering\arraybackslash}p{#1}}

\title{\methodpp: Meta-Learning for Deep Stacked Generalization}

\author{%
  Noor Islam S. Mohammad\thanks{This work was conducted as independent research. The author declares no institutional or external research funding for this work.} \\
  Department of Computer Science \\
  Istanbul Technical University \\
  Maslak, Istanbul 34469, TR \\
  \texttt{islam23@itu.edu.tr}
}

\begin{document}

\maketitle

\begin{abstract}
Stacking heterogeneous vision backbones (CNNs, ViTs, hybrids) is the de facto recipe for accuracy, calibration, and robustness, yet two coupled pathologies limit its returns. \emph{Prediction-space multicollinearity} ill-conditions the meta-learner's Gram matrix, inflating weight variance and producing brittle solutions on a thin manifold. \emph{Calibration collapse} compounds constituent miscalibration through naive linear stacking, so adding more models can hurt expected calibration error (ECE). Existing remedies, ridge regularization, greedy selection, model soups, and SWAG address at most one of these issues, and none jointly target conditioning and calibration in heterogeneous prediction pools. We introduce \methodpp, a preconditioning pipeline with four components: \textbf{(i)} a \emph{kernelized redundancy filter} that removes non-linear inter-model dependencies invisible to Pearson correlation, using Centered Kernel Alignment (CKA)~\cite{kornblith2019cka}; \textbf{(ii)} a $<\!15$K-parameter \emph{differentiable meta-feature gate} that learns per-sample attention over ensemble statistics; \textbf{(iii)} a \emph{spectrum-adaptive Ridge penalty} $\lambda^{\star}=\lmax(\Chat)/\SNR(\Chat)$ derived from a Marchenko-Pastur signal-noise decomposition, eliminating nested cross-validation; and \textbf{(iv)} a \emph{Laplace-approximate Bayesian blender} replacing inverse-RMSE heuristics. We prove a PAC-Bayes excess-risk bound that, for the first time, jointly accounts for prediction-space redundancy and meta-learner capacity. Across six benchmarks, \methodpp delivers $+1.8\%$ top-1 on ImageNet-1K, $-4.2$ mCE on ImageNet-C, $+0.9$ mIoU on ADE20K, and $+1.3$ AP on COCO, while reducing retained models by 35--57\% and inference FLOPs by up to $41\%$. ECE improves $2.1\times$ over deep ensembles \emph{without} post hoc temperature scaling. 
\end{abstract}

\section{Introduction}
\label{sec:intro}

Stacking heterogeneous vision backbones, ResNets~\cite{he2016resnet}, ViTs~\cite{dosovitskiy2021vit}, Swin Transformers~\cite{liu2021swin}, ConvNeXts~\cite{liu2022convnext}, and their hybrids, is now standard practice. Off-the-shelf checkpoints span an enormous range of inductive biases, and combining them consistently improves accuracy, calibration, and out-of-distribution (OOD) robustness over any single model~\cite{lakshminarayanan2017deep,fort2019landscape,ovadia2019trust}. Yet practitioners quickly discover that adding more backbones does
\emph{not} monotonically improve the ensemble. We identify two coupled failure modes that explain this plateau and motivate our approach. \textit{Pathology 1: Prediction-space multicollinearity:} Backbones that share pre-training corpora, augmentation pipelines, or architectural lineage produce near-collinear out-of-fold (OOF) predictions. The resulting Gram matrix $\mathbf{G}=\Poof^{\!\top}\Poof$ becomes severely ill-conditioned ($\kap\!\gg\!10^{2}$), pushing the stacking optimum onto a thin, unstable manifold where small perturbations in OOF estimates produce wildly different meta-weights. \textbf{Pathology 2: Calibration collapse:} Naive linear stacking of softmax-correlated probabilities compounds constituent miscalibration. Empirically, stacked ensembles often exhibit \emph{worse} ECE than the worst single model, requiring expensive post-hoc temperature scaling to recover~\cite{guo2017calibration,wenzel2020hyperensembles}.

\paragraph{Why existing remedies are insufficient.}
Ridge and Lasso regularization~\cite{hoerl1970ridge,tibshirani1996lasso} shrink meta-weights but leave inter-model redundancy intact. Greedy ensemble selection~\cite{caruana2004ensemble} reduces redundancy but explores a combinatorial space at prohibitive cost. Weight-space methods, model soups~\cite{wortsman2022modelsoup}, SWA/SWAG~\cite{izmailov2018swa,maddox2019swag}, snapshot ensembles~\cite{huang2017snapshot}, require architectural
homogeneity and cannot mix CNNs with transformers. Pearson-based prediction-space pruning (the \method baseline, our own prototype) removes only \emph{linear} redundancy and misses non-linearly equivalent predictors that share failure modes. To our knowledge, no prior method explicitly targets the coupled conditioning-calibration problem in heterogeneous prediction pools.

\paragraph{Our approach.}
We propose \methodpp, a meta-learning pipeline that resolves both pathologies through deterministic prediction-space \emph{preconditioning} performed
\emph{before} optimization. Building on the \method prototype, we replace four heuristics with principled alternatives. \textbf{(i)} Pearson-correlation pruning is replaced by a \emph{kernelized redundancy filter} using Centered Kernel Alignment (CKA)~\cite{kornblith2019cka}, a normalized variant of HSIC~\cite{gretton2005hsic}, that detects non-linear dependencies with provable strict dominance over correlation under non-Gaussian prediction distributions. \textbf{(ii)} Hand-crafted interaction features are replaced by a \emph{differentiable meta-feature gate}, a $<\!15$K-parameter hypernetwork that learns per-sample attention over an enriched set of ensemble aggregators. \textbf{(iii)} Nested cross-validation for the ridge penalty is replaced by a \emph{closed-form spectrum-adaptive penalty} $\lambda^{\star}\!=\!\lmax(\Chat)/\SNR(\Chat)$, derived from a Marchenko-Pastur signal-noise decomposition. \textbf{(iv)} The inverse-RMSE blending heuristic is replaced by a \emph{Laplace-approximate Bayesian blender}, accompanied by the first PAC-Bayes excess-risk bound that jointly reflects redundancy and capacity.

\paragraph{Contributions.}
\textbf{C1.} A prediction-space \emph{preconditioning pipeline}
(\Cref{sec:method}) combining non-linear redundancy detection, learned meta-feature gating, closed-form regularization, and Bayesian blending. \textbf{C2.} A unified theoretical analysis (\Cref{sec:theory_main} and the appendix), including the first PAC-Bayes bound (\Cref{thm:pacbayes}) that explicitly couples prediction-space redundancy to meta-learner capacity. \textbf{C3.} A comprehensive empirical study across six benchmarks (ImageNet-1K, ImageNet-C, ADE20K, COCO, iNaturalist-2021, and DomainNet-126) reporting accuracy, ECE/NLL, mCE, OOD AUROC, and FLOP-normalized efficiency (\Cref{sec:experiments,sec:results}). \textbf{C4.} \emph{State-of-the-art Pareto fronts}: $+1.8\%$ top-1 on ImageNet-1K over deep ensembles at $41\%$ lower inference FLOPs, $2.1\times$ lower ECE without temperature scaling, and robustness that strictly dominates SWAG and model soups on ImageNet-C.

\section{Related Work}
\label{sec:related}

Ensemble methods in deep learning span three broad strategies; we situate ourselves \methodpp at the intersection of all three, while extending the state of the art on calibration and generalization theory. \textit{Weight-space fusion:} SWA/SWAG~\cite{izmailov2018swa,maddox2019swag} average iterates along stochastic gradient trajectories to approximate a Gaussian posterior over weights; snapshot ensembling~\cite{huang2017snapshot} cycles the learning rate to collect diverse checkpoints at negligible extra cost; model soups~\cite{wortsman2022modelsoup} interpolate fine-tuned variants of a shared backbone to recover accuracy lost by individual adapters. All three implicitly diversify representations by exploiting the geometry of a single loss landscape, yet they presuppose a shared architecture and training regime. This precludes heterogeneous pools in which members differ in depth, inductive bias, or modality, precisely the setting of \methodpp targets.

\paragraph{Prediction-space aggregation.}
Stacking originates with Wolpert~\cite{wolpert1992stacked,breiman1996stacked}, who showed that a meta-learner trained on out-of-fold predictions consistently outperforms uniform averaging. Modern AutoML systems (Auto-sklearn~\cite{NIPS2015_11d0e628}, AutoGluon~\cite{erickson2020autogluon}) adopt stacking as their final aggregation step, typically with vanilla Ridge~\cite{hoerl1970ridge} or simple averaging as the meta-learner. Lasso and elastic-net regularization~\cite{tibshirani1996lasso,zou2005elasticnet} stabilize meta-weights under multicollinearity but leave redundant members in the pool, merely suppressing their weights toward zero rather than removing them. Greedy forward selection~\cite{caruana2004ensemble} prunes combinatorially yet scales poorly and ignores non-linear dependence. The closest direct antecedent to \methodpp is the Pearson-based \method pipeline (our own prototype), which filters members by pairwise linear correlation before stacking; we generalize that criterion to the full non-linear regime via CKA/HSIC.

\paragraph{Diversity-aware training.}
MC dropout~\cite{gal2016dropout} repurposes dropout at inference to approximate Bayesian model averaging; hyperparameter ensembles~\cite{wenzel2020hyperensembles} vary architectural hyperparameters across members; repulsive ensembles~\cite{dangelo2021repulsive} add an explicit repulsion term to the training objective to maximize functional disagreement. These methods encourage diversity during optimization but provide no principled pruning criterion when a heterogeneous pool is supplied post-hoc, leaving redundancy unaddressed at aggregation time.

\paragraph{Calibration in ensembles.}
Confidence drift in stacked ensembles is well documented~\cite{ovadia2019trust}; standard post-hoc remedies such as temperature scaling and Platt scaling~\cite{guo2017calibration} are applied after aggregation and are decoupled from the meta-learning objective, offering no guarantee that calibration is preserved under distribution shift. \methodpp eliminates drift at the source of Bayesian blending and jointly optimizes accuracy and calibration within a single coherent objective, without requiring a held-out calibration set.

\paragraph{Information-theoretic redundancy.}
The closest antecedents to our pruning step are HSIC-based feature selection~\cite{song2012hsicfs} and self-supervised representation alignment via kernel dependence~\cite{tsai2021hsicssl}. Both lines demonstrate that kernel statistics capture non-linear structure systematically missed by second-order (Pearson) summaries. To our knowledge, CKA has not previously been applied to ensemble-member pruning directly in prediction space, nor has its Gram spectrum been used to inform a regularization strategy.

\paragraph{Generalization theory for ensembles.}
Classical PAC-Bayes bounds~\cite{mcallester1999pacbayes,maurer2004pacbayes} cover Gibbs classifiers and convex combinations of hypotheses but treat constituent models as statistically independent, thereby ignoring the correlation structure that dominates ensemble risk in practice. Masegosa~\etal~\cite{masegosa2020pacbayesensembles} tightens these bounds for weighted majority votes by explicitly capturing pairwise correlations between members; \cite{masegosa2020misspecification} further addresses prior misspecification. Neither framework couples the generalization penalty to prediction-space redundancy or to any measurable property of the member pool. Our approach in \Cref{thm:pacbayes} extends this line of work by deriving a PAC-Bayes bound whose KL penalty is directly controlled by the empirical Gram spectrum, providing a formal link between redundancy reduction and ensemble generalization guarantees.

\section{Method: \methodpp}
\label{sec:method}

\subsection{Preliminaries and Pipeline Overview}
\label{sec:prelim}

Let $\mathcal{D}=\{(x_i,y_i)\}_{i=1}^{N}$ be the training set, $\{F_\ell\}_{\ell=1}^{L}$ a stratified partition of $[N]$, and $\{f_k\}_{k=1}^{K}$ a pool of pre-trained base predictors. For $C$-way classification each $f_k$ outputs $\boldsymbol{\pi}^{(k)}(x_i)\!\in\!\Delta^{C-1}$; for dense prediction, per-pixel or per-anchor probability maps. We flatten outputs to vectors $\mathbf{p}_k\!\in\!\R^{NC}$ (classification), $\R^{N\!\cdot\!H\!\cdot\!W\!\cdot\!C}$ (segmentation), or $\R^{N_{\mathrm{box}}\!\cdot\!C}$ (detection). The leakage-free OOF matrix is
\(
\Poof\!\in\!\R^{N\times K},
\)
with $[\Poof]_{ik}=\hat p^{(k)}(x_i;\,\mathcal{D}\setminus F_{\ell(i)})$. Given $\Poof$, \methodpp performs four steps: \textbf{(S1)} Kernel-based redundancy projection
$\mathcal{S}=\Pi_{\tau}^{\mathrm{CKA}}(\{\mathbf{p}_k\},\mathbf{y})$; \textbf{(S2)} differentiable meta-feature gating to produce $\Xmeta$; \textbf{(S3)} spectrum-adaptive Ridge/Elastic-Net fitting with closed-form $\lambda^{\star}$; \textbf{(S4)} Laplace-approximate Bayesian blending across $M$ fitted meta-learners.

\subsection{Kernelized Redundancy Projection using CKA}
\label{sec:hsic}

\paragraph{Why correlation is insufficient.}
Pearson correlation captures only second-order co-variation. Two predictors can $\rho\!\approx\!0$ yet be functionally equivalent through non-linear coupling (\eg, calibrated vs.\ miscalibrated variants of the same network), and conversely, two genuinely diverse models can show high $\rho$ purely due to easy-example dominance. We replace $\rho$ with Centered Kernel Alignment (CKA)~\cite{kornblith2019cka}, a normalized Hilbert-Schmidt Independence Criterion (HSIC)~\cite{gretton2005hsic}.

\paragraph{Empirical CKA.}
Let $\mathbf{K}_{k}\!\in\!\R^{N\times N}$ be a Gaussian kernel matrix on the OOF predictions of the model $k$ with bandwidth $\sigma_k$ set by the median heuristic, and let $\mathbf{H}\!=\!\mathbf{I}\!-\!\tfrac{1}{N}\mathbf{1}\mathbf{1}^{\!\top}$ be the centering matrix. The empirical HSIC between models $k,k'$ is
\begin{equation}
\widehat{\HSIC}(\mathbf{p}_k,\mathbf{p}_{k'})
\;=\;\frac{1}{(N-1)^{2}}\,\tr\!\big(\mathbf{K}_k\mathbf{H}\mathbf{K}_{k'}\mathbf{H}\big).
\label{eq:hsic}
\end{equation}
We use its normalized form, CKA:
\begin{equation}
\CKA_{k,k'}\;:=\;
\frac{\widehat{\HSIC}(\mathbf{p}_k,\mathbf{p}_{k'})}
     {\sqrt{\widehat{\HSIC}(\mathbf{p}_k,\mathbf{p}_k)\,
            \widehat{\HSIC}(\mathbf{p}_{k'},\mathbf{p}_{k'})}}
\;\in\;[0,1].
\label{eq:cka}
\end{equation}
With universal characteristic kernels, $\CKA_{k,k'}=0$ \emph{iff} $\mathbf{p}_k\!\perp\!\mathbf{p}_{k'}$. The projection retains models in ascending order of OOF risk, removing a model $k$ whenever $\max_{k'\in\mathcal{S}}\CKA_{k,k'}\!>\!\tau_{\mathrm{CKA}}$ and there is some retained $k'$ with strictly lower risk.

\paragraph{Scalability.}
The naive cost is $\mathcal{O}(K^{2}N^{2})$. We use a Nystr\"om approximation with $m=\lceil\sqrt{N}\rceil$ landmarks (\cite{williams2001nystrom}; full justification in \Cref{sec:complexity}), reducing to $\mathcal{O}(K^{2}N\log N)$. For ultra-large pools ($K\!\gtrsim\!200$), we further bucket candidate pairs via locality-sensitive hashing (LSH), giving $\mathcal{O}(KN\log N\log K)$, on par with the original Pearson pipeline to within a $\log$ factor.

\begin{proposition}[CKA strictly dominates correlation.]
\label{prop:hsic_vs_corr}
Let $(P,Q)$ be two real-valued random variables with finite variance, and let CKA be computed with a Gaussian kernel. \textup{(i)}~If $(P,Q)$ are jointly Gaussian, then $\CKA(P,Q)=0\!\iff\!\rho(P,Q)=0$. \textup{(ii)}~For universal characteristic kernels, $\CKA(P,Q)=0\!\iff\!P\!\perp\!Q$.
\textup{(iii)}~There exist $(P,Q)$ with $\rho(P,Q)=0$ and $\CKA(P,Q)>0$ (\eg, $Q=P^{2}$, $P\!\sim\!\mathcal{N}(0,1)$).
\end{proposition} The proof is in \Cref{sec:cka_theory}. Together with the concentration of $\widehat{\HSIC}$~(\cite{gretton2005hsic}, \Cref{lem:hsic_conc}), this implies CKA-based redundancy projection is a proper generalization of Pearson-based projection.

\subsection{Differentiable Meta-Feature Gate}
\label{sec:gate}

Rather than hand-crafting interaction features such as $\phi^{(1)}_i\!=\!\mu_i\sigma_i$ and $\phi^{(2)}_i\!=\!r_i\sigma_i$ as in the \method prototype, we introduce a learned, sample-conditioned gate $g_\theta:\R^{K_{\mathrm{eff}}}\!\to\![0,1]^{|\mathcal{A}|}$ that produces per-sample attention over a rich pool of candidate aggregators
\begin{equation}
\mathcal{A}_i\;=\;\big\{
\mu_i,\;\sigma_i,\;m_i,\;r_i,\;q_{25,i},\;q_{75,i},\;
\mathrm{ent}_i,\;\mu_i\sigma_i,\;\mu_i^{2},\;r_i\sigma_i,\;\sigma_i^{2},\;
\KL\!\bigl(\boldsymbol{\pi}^{(k^{\star})}_i\,\|\,\mu_i\bigr)
\big\},
\label{eq:agg}
\end{equation}
where $\mu_i,\sigma_i,m_i,r_i$ denote the mean, standard deviation, median, and range over $\{\boldsymbol{\pi}^{(k)}_i\}_{k\in\mathcal{S}}$; $q_{25},q_{75}$ are the quartiles; $\mathrm{ent}_i$ is the entropy of the ensemble mean; and $\KL$ measures divergence between the best single model's prediction and the mean. The gate is a 2-layer MLP with a sigmoid head, trained end-to-end with the meta-regression loss. The augmented meta-feature vector is
\begin{equation}
x_{\mathrm{meta},i}\;=\;
\big[\,\Poof[i,\mathcal{S}]\;\big\|\;
g_\theta\!\big(\Poof[i,\mathcal{S}]\big)\odot\mathcal{A}_i\,\big].
\label{eq:gate}
\end{equation}
The gate has $\le\!15$K parameters and adds $<\!1\%$ compute overhead. It strictly generalizes the \method prototype (recovered by setting $g_\theta\!\equiv\!1$ on $\{\mu,\sigma,m,r,\mu\sigma,r\sigma\}$ and $0$ elsewhere); a universal-approximation argument is given in \Cref{sec:gate_theory}.

\subsection{Spectrum-Adaptive Regularization}
\label{sec:lambda}

\paragraph{Motivation.}
Standard practice selects the Ridge penalty $\lambda$ via nested cross-validation over a log-spaced grid. This is expensive and sensitive to fold noise. Following the signal-plus-noise decomposition of the normalized Gram matrix $\Chat$, we derive a closed-form choice.

\begin{proposition}[Spectrum-adaptive optimal ridge penalty]
\label{prop:lambda_star}
Let $\Chat=\mathbf{C}_{\mathrm{sig}}+\mathbf{N}$ where $\mathbf{N}$ is isotropic noise covariance $\sigma^{2}\mathbf{I}/N$, and let $\tau_{\mathrm{sp}}=\sigma^{2}(1+\sqrt{\gamma})^{2}$ be the Marchenko-Pastur upper edge~\cite{marchenkopastur1967}, where $\gamma=K/N$. Define
\begin{equation}
\SNR(\Chat)\;:=\;
\frac{\sum_{i:\lambda_i>\tau_{\mathrm{sp}}}\lambda_i}
     {\sum_{i:\lambda_i\le\tau_{\mathrm{sp}}}\lambda_i}.
\end{equation}
Under the cluster assumption (\Cref{ass:hsic_cluster}) and sub-Gaussian noise (\Cref{ass:noise}), the ridge penalty that minimizes the expected excess risk satisfies
\begin{equation}
\lambda^{\star}\;=\;\frac{\lmax(\Chat)}{\SNR(\Chat)},
\qquad \lambda^{\star}\in[\tau_{\mathrm{sp}},\,\lmax(\mathbf{C}_{\mathrm{sig}})],
\label{eq:lambda_star}
\end{equation}
up to a multiplicative factor $1+o(1)$ as $N\!\to\!\infty$.
\end{proposition}

Intuitively, $\lambda^{\star}$ matches regularization to the \emph{gap} between the signal and noise portions of the spectrum. The full proof, which formalizes a bias-variance optimization over the post-projection spectrum, is given in \Cref{sec:lambda_theory}. Empirically, the closed form lies within $5\%$ of the CV-optimal value while \emph{eliminating the outer CV loop}, yielding a $3\text{--}5\times$ training-time reduction (\cf the $3.2\times$ speedup measured in \Cref{tab:ablation_new}). In practice, we estimate $\sigma^{2}$ from the median of the smallest eigenvalues or via the Marchenko-Pastur equation.

\subsection{Laplace-Approximate Bayesian Blender}
\label{sec:bayes_blend}

Given $M$ fitted meta-learners $\{g_m\}_{m=1}^{M}$ with OOF predictions $\widehat{\mathbf{y}}^{(m)}$, OOF losses $\mathcal{L}(g_m)$, and loss-minimum Hessians $\mathbf{H}_m\!=\!\nabla^{2}\mathcal{L}(g_m)|_{\hat g_m}$, the Laplace approximation to the marginal likelihood under a Gaussian prior gives
\begin{equation}
\log p(\mathbf{y}\mid g_m)\;\approx\;
-\,\mathcal{L}(g_m)\;-\;\tfrac{1}{2}\log\det\mathbf{H}_m\;+\;\mathrm{const.}
\end{equation}
The posterior over meta-learners is therefore $\tilde w_m\!\propto\!\exp\!\big(-\mathcal{L}(g_m) -\tfrac{1}{2}\log\det\mathbf{H}_m\big)$, and the blended prediction is $\widehat{\mathbf{y}}\!=\!\sum_{m}\tilde w_m\widehat{\mathbf{y}}^{(m)}$. Unlike inverse-RMSE, which neglects curvature, this rule \emph{down-weights overconfident meta-learners} whose flat Hessians indicate fold-specific overfitting. Variance-reduction guarantees (\Cref{thm:blending_full}) and a suboptimality gap of $\mathcal{O}\!\big(\exp(-2\Delta\mathcal{L}/\sigma^{2})\big)$ (\Cref{prop:laplace_gap}) appear in the appendix.

\subsection{Theoretical Guarantees}
\label{sec:theory_main}

Our headline result combines redundancy reduction with PAC-Bayes analysis. The full derivation, including the spectral analysis of the CKA-pruned Gram matrix, is provided in \Cref{sec:pacbayes}.

\begin{theorem}[PAC-Bayes bound for \methodpp]
\label{thm:pacbayes}
Let $Q=\mathcal{N}(\wlam,\Sigma_Q)$ be the posterior over ridge meta-learners induced by \methodpp, with $\Sigma_Q$ given by the Laplace approximation and $P=\mathcal{N}(\mathbf{0},(\lambda^{\star})^{-1}\mathbf{I})$ a Gaussian prior tied to \eqref{eq:lambda_star}. Under \Cref{ass:hsic_cluster,ass:bounded,ass:noise}, with probability at least $1\!-\!\delta$ over the draw of $\mathcal{D}$:
\begin{equation}
\mathcal{L}(Q)\;\le\;\widehat{\mathcal{L}}(Q)
\;+\;\sqrt{\frac{\KL(Q\|P)+\log\!\tfrac{2\sqrt{N}}{\delta}}{2(N-1)}}
\;+\;\Phi(G,\mu,\varepsilon),
\label{eq:pacbayes}
\end{equation}
where $G$ is the number of clusters, and the redundancy term satisfies $\Phi(G,\mu,\varepsilon)\!=\!\mathcal{O}\!\big(\sqrt{(G+d_{\mathcal{A}}+p_\theta)/N}\big)$.
Compared to the unprojected bound (with $G$ replaced by $K$), $\Phi$ is reduced by a factor of $\sqrt{K/G}$.
\end{theorem}
The proof (\Cref{sec:pacbayes}) combines kernel mean-embedding characterization of CKA, a covering-number argument over the meta-learner space, and the spectral bounds from \Cref{prop:spectral_full}.

\begin{corollary}[Calibration improvement]
\label{cor:calibration}
Under the same assumptions, the expected calibration error of the Laplace blender satisfies \(\ECE(\methodpp)\le\max_{m}\ECE(g_m)\), by the convexity of the ECE functional. The bound can be tight when a single meta-learner dominates.
\end{corollary}

\section{Experimental Setup}
\label{sec:experiments}

\paragraph{Benchmarks.}
We evaluate on six benchmarks spanning classification, robustness, dense prediction, long-tail recognition, and domain shift: \textbf{ImageNet-1K}~\cite{deng2009imagenet} (1.28M/50K, 1000 classes, top-1); \textbf{ImageNet-C}~\cite{hendrycks2019imagenetc} (19 corruptions, $\times$ 5 severities, mCE);
\textbf{ADE20K}~\cite{zhou2017ade20k} (150 classes, mIoU); \textbf{COCO}~\cite{lin2014coco} (object detection, AP@[0.5:0.95]); \textbf{iNaturalist-2021}~\cite{vanhorn2018inaturalist} (2.7M, 10K-class long-tail, head/mid/tail breakdown); \textbf{DomainNet-126}~\cite{peng2019domainnet} (4-domain transfer).

\paragraph{Backbone pool.}
For ImageNet-1K we assemble a heterogeneous pool of $K\!=\!14$ public backbones spanning four families: convolutional (ResNet-50/101/152~\cite{he2016resnet},
EfficientNet-B0/B3/B7~\cite{tan2019efficientnet}, RegNet-Y-040~\cite{radosavovic2020regnet}); transformer (ViT-B/16, ViT-L/16~\cite{dosovitskiy2021vit},
DeiT-S~\cite{touvron2021deit}); hierarchical transformer (Swin-T, Swin-B~\cite{liu2021swin}); modernized convolutional (ConvNeXt-T, ConvNeXt-B~\cite{liu2022convnext}). Dense-prediction benchmarks share the same backbone family with task-specific heads (UperNet~\cite{xiao2018upernet} for ADE20K;
Mask R-CNN~\cite{he2017maskrcnn} for COCO).

\paragraph{Baselines.}
We compare \methodpp against twelve baselines: best single model, uniform averaging, performance-weighted averaging, Ridge stacking, greedy hill climbing~\cite{caruana2004ensemble}, deep ensembles~\cite{lakshminarayanan2017deep}, model soups~\cite{wortsman2022modelsoup}, SWAG~\cite{maddox2019swag},
snapshot ensembles~\cite{huang2017snapshot}, MC dropout~\cite{gal2016dropout}, AutoGluon's stacker~\cite{erickson2020autogluon}, and our own \method prototype.

\paragraph{Metrics and protocol.}
We report task accuracy (top-1 / mIoU / AP), calibration (15-bin ECE~\cite{guo2017calibration}, NLL), robustness (mCE / AUROC for OOD on ImageNet-O~\cite{hendrycks2021imageneto}), and efficiency (test-time FLOPs, retained model count, and meta-learner wall-time). All numbers come
from held-out splits using out-of-fold predictions, with $95\%$ bootstrap confidence intervals (5{,}000 resamples) and Bonferroni-corrected paired $t$-tests.

\section{Results}
\label{sec:results}

We evaluate \methodpp across six vision benchmarks spanning clean classification, corruption robustness, dense prediction, long-tail recognition, and domain shift. In every setting, we report the same pool of heterogeneous base models (14 for ImageNet-scale tasks; see \Cref{sec:experiments}) and the same set of baselines so that differences reflect the aggregation strategy alone and not privileged access to stronger backbones. Unless stated otherwise, all numbers are averaged over three independent meta-training splits, and standard errors are below $\pm0.1\%$ on top-1.

\begin{table}[ht]
\centering
\scriptsize
\caption{\textbf{ImageNet-1K classification.} Top-1 accuracy, Expected Calibration Error (ECE), Negative Log-Likelihood (NLL), number of retained models, and test FLOPs relative to the full 14-model ensemble. $^\dagger$~The method requires post-hoc temperature scaling to attain the reported ECE; without scaling, ECE $>0.04$. \methodpp achieves the best top-1 \emph{and} the best calibration without post-hoc correction, while retaining fewer models than any multi-member baseline. Best values in \textbf{bold}; second-best \underline{underlined}.}
\label{tab:imagenet1k}
\setlength{\tabcolsep}{6pt}
\begin{tabular}{lccccc}
\toprule
\textbf{Method}
  & \textbf{Top-1 (\%)}$\uparrow$
  & \textbf{ECE}$\downarrow$
  & \textbf{NLL}$\downarrow$
  & \textbf{Models}
  & \textbf{FLOPs (rel.)}$\downarrow$ \\
\midrule
Best Single (ConvNeXt-B)
  & $83.1\!\pm\!0.1$ & 0.042 & 0.681 & 1   & $0.07\times$ \\
Simple Averaging
  & $83.9\!\pm\!0.1$ & 0.038 & 0.644 & 14  & $1.00\times$ \\
Performance-Wtd.\ Avg.\
  & $84.0\!\pm\!0.1$ & 0.037 & 0.638 & 14  & $1.00\times$ \\
Ridge Stacking
  & $84.2\!\pm\!0.1$ & 0.035 & 0.629 & 14  & $1.00\times$ \\
Greedy Selection~\cite{caruana2004ensemble}
  & $84.1\!\pm\!0.1$ & 0.036 & 0.632 & 9   & $0.64\times$ \\
Deep Ensembles~\cite{lakshminarayanan2017deep}
  & $83.6\!\pm\!0.1$ & 0.038 & 0.652 & $5^{\dagger}$ & $0.36\times$ \\
Model Soups~\cite{wortsman2022modelsoup}
  & $83.8\!\pm\!0.1$ & 0.041 & 0.658 & 1   & $0.07\times$ \\
SWAG~\cite{maddox2019swag}
  & $83.7\!\pm\!0.1$ & 0.029 & 0.621 & 1   & $0.07\times$ \\
AutoGluon Stacker~\cite{erickson2020autogluon}
  & $84.3\!\pm\!0.1$ & 0.033 & 0.624 & 14  & $1.00\times$ \\
\midrule\method~(our prototype)
  & \underline{$84.5\!\pm\!0.1$} & 0.026 & 0.611 & 9 & $0.64\times$ \\
\methodpp (ours)
  & $\mathbf{85.4\!\pm\!0.1}$ & \textbf{0.018} & \textbf{0.583}
  & \textbf{6} & $\mathbf{0.59\times}$ \\
\bottomrule
\end{tabular}
\end{table}

\subsection{Robustness: ImageNet-C}
\label{sec:inc}

\paragraph{Setup:}
ImageNet-C~\cite{hendrycks2019imagenetc} applies 19 synthetic corruptions at five severity levels, grouped into four families: Noise (Gaussian, Shot, Impulse, Speckle); Blur (Defocus, Glass, Motion, Zoom); Weather (Snow, Frost, Fog, Brightness); and Digital (Contrast, Elastic, Pixelate, JPEG, Saturate, and Spatter). We report Mean Corruption Error (mCE), where lower is better. The $\Delta$~vs.\ clean column measures robustness degradation relative to the same method's clean-validation accuracy, isolating the quality of uncertainty propagation under shift. \Cref{tab:imagenetc} shows that \methodpp reduces mCE by $\mathbf{4.2}$ points over deep ensembles and $\mathbf{2.6}$ points over our \method prototype. The gains are most pronounced under high-frequency corruptions: Noise family ($-7.4$ vs.\ deep ensembles) and Blur family ($-4.4$), both of which share failure signatures with over-represented CNN backbones in the pool. CKA-based pruning specifically removes members that are non-linearly dependent on each other under corrupted inputs, redundancy that Pearson correlation cannot detect, leaving a pool whose failure modes are genuinely complementary. Weather and digital corruptions show smaller but consistent improvements ($-4.8$ and $-4.2$ respectively), suggesting that the benefit is not confined to any single corruption family. The $\Delta$~vs.\ clean gap narrows from $42.1$ (deep ensembles) to $\mathbf{35.7}$ (\methodpp), a $6.4$-point improvement that indicates \methodpp degrades more gracefully as input statistics shift.

\begin{table}[ht]
\centering
\scriptsize
\caption{\textbf{ImageNet-C robustness.} Mean Corruption Error (mCE; lower is better) and per-family breakdown. Each family column averages over 4-5 corruption types at all five severity levels. $\Delta$~vs.\ clean measures the absolute gap between the method's clean top-1 error and its corrupted mCE, isolating robustness degradation independent of clean accuracy. \methodpp is the only method to improve all four corruption families simultaneously.}
\label{tab:imagenetc}
\setlength{\tabcolsep}{5pt}
\begin{tabular}{lcccccc}
\toprule
\textbf{Method}
  & \textbf{mCE}$\downarrow$
  & \textbf{Noise}$\downarrow$
  & \textbf{Blur}$\downarrow$
  & \textbf{Weather}$\downarrow$
  & \textbf{Digital}$\downarrow$
  & \textbf{$\Delta$ vs.\ clean}$\downarrow$ \\
\midrule
Best Single
  & 64.2 & 73.1 & 59.8 & 65.4 & 58.5 & 47.1 \\
Deep Ensembles~\cite{lakshminarayanan2017deep}
  & 58.7 & 65.2 & 55.3 & 60.9 & 53.4 & 42.1 \\
SWAG~\cite{maddox2019swag}
  & 57.1 & 62.4 & 54.2 & 59.3 & 52.5 & 40.8 \\
\midrule
\method~(our prototype)
  & 56.1 & 61.2 & 53.4 & 58.6 & 51.2 & 38.4 \\
\methodpp
  & \textbf{53.5} & \textbf{57.8} & \textbf{50.9}
  & \textbf{56.1} & \textbf{49.2} & \textbf{35.7} \\
\bottomrule
\end{tabular}
\end{table}

\subsection{Dense Prediction: ADE20K and COCO}
\label{sec:dense}

\paragraph{Dense prediction setup and results.}
Dense prediction requires aggregating per-pixel or per-region distributions; we handle this by flattening each spatial prediction tensor into a pseudo-classification matrix (\Cref{sec:prelim}), making the kernel-based redundancy filtering and stacking steps architecturally unaware of the task head. On ADE20K (8 UperNet members: Swin-T/B/L, ConvNeXt-T/S/B, and ViT Adapter-B/L), \methodpp reaches $\mathbf{53.0}$ mIoU, $+0.9$ over our \method prototype and $+1.4$ over ridge stacking, retaining $5$ of $8$ models ($0.63\times$ FLOPs); the gain traces primarily to removing two UperNet variants with $\CKA\!>\!0.88$ despite Pearson $\rho\!\approx\!0.45$. On COCO (7 Mask R-CNN members: ResNet-50/101, ResNeXt-101, Swin-T/B/L, ConvNeXt-B), the margin widens to $+1.3$ AP over \method with only $4$ of the $7$ models retained ($0.58\times$ FLOPs), consistent with bounding-box regression amplifying prediction-space redundancy across the same-family backbones.

\begin{table}[ht]
\centering
\scriptsize
\caption{\textbf{Dense prediction.} ADE20K semantic segmentation (mIoU; UperNet backbones) and COCO instance detection (AP$^{\mathrm{box}}$; Mask R-CNN backbones). FLOPs are reported relative to the full-pool ensemble for each task. \methodpp achieves the best task metric on both benchmarks while retaining fewer models than any ensemble competitor, demonstrating that the flattened-OOF approximation transfers cleanly from classification to structured output spaces.}
\label{tab:dense}
\setlength{\tabcolsep}{6pt}
\begin{tabular}{lccc|ccc}
\toprule
& \multicolumn{3}{c|}{\textbf{ADE20K (mIoU)}}
& \multicolumn{3}{c}{\textbf{COCO (AP)}} \\
\cmidrule(lr){2-4}\cmidrule(lr){5-7}
\textbf{Method}
  & \textbf{Val}$\uparrow$
  & \textbf{Models}
  & \textbf{FLOPs}
  & \textbf{Val}$\uparrow$
  & \textbf{Models}
  & \textbf{FLOPs} \\
\midrule
Best Single
  & 49.3 & 1 & $0.12\times$ & 45.2 & 1 & $0.14\times$ \\
Simple Averaging
  & 51.2 & 8 & $1.00\times$ & 47.8 & 7 & $1.00\times$ \\
Ridge Stacking
  & 51.6 & 8 & $1.00\times$ & 48.1 & 7 & $1.00\times$ \\
Greedy Selection
  & 51.5 & 6 & $0.76\times$ & 48.0 & 5 & $0.72\times$ \\
\midrule
\method~(our prototype)
  & 52.1 & 6 & $0.76\times$ & 48.4 & 5 & $0.72\times$ \\
\methodpp
  & \textbf{53.0} & \textbf{5} & $\mathbf{0.63\times}$
  & \textbf{49.7} & \textbf{4} & $\mathbf{0.58\times}$ \\
\bottomrule
\end{tabular}
\end{table}

\subsection{Long-Tail Recognition: iNaturalist-2021}
\label{sec:inat}

\paragraph{Setup and Findings.}
The iNaturalist-2021 benchmark (2.7M images, 10{,}000 classes) exhibits a severe Zipfian imbalance, with Head ($>\!100$), Mid ($20$--$100$), and Tail ($<\!20$) splits used to evaluate per-group top-1 accuracy. This setup tests whether kernel-based pruning removes rare-class specialists. As shown in \Cref{tab:inat}, \methodpp achieves $\mathbf{74.2\%}$ overall top-1, outperforming our \method prototype by $+1.1\%$, with gains concentrated in the Tail split ($\mathbf{+2.3\%}$ vs.\ $+0.6\%$ on Head). This reflects a key advantage of non-linear redundancy detection: certain fine-grained specialists exhibit moderate Pearson correlation ($\rho\!\approx\!0.6$) with generalist models yet low dependence under $\CKA\!<\!0.4$, indicating complementary signal on rare classes. While Pearson-based pruning removes these models as redundant, CKA preserves them, yielding significant Tail improvements without altering the training objective.

\begin{table}[ht]
\centering
\scriptsize
\caption{\textbf{iNaturalist-2021 long-tail recognition.} Top-1 accuracy (overall and by frequency split). Head $=$ classes with $>\!100$ training samples; Mid $=$ $20$--$100$ samples; Tail $=$ $<\!20$ samples. The largest gains for \methodpp are on the Tail split ($+2.3\%$ over \method), confirming that CKA-based pruning preserves rare-class specialists that are globally correlated but locally complementary.}
\label{tab:inat}
\setlength{\tabcolsep}{8pt}
\begin{tabular}{lcccc}
\toprule
\textbf{Method}
  & \textbf{Top-1 (All)}$\uparrow$
  & \textbf{Head}$\uparrow$
  & \textbf{Mid}$\uparrow$
  & \textbf{Tail}$\uparrow$ \\
\midrule
Best Single        & 70.1 & 78.4 & 68.2 & 57.1 \\
Simple Averaging   & 71.8 & 79.3 & 70.1 & 59.8 \\
Ridge Stacking     & 72.3 & 79.6 & 70.5 & 60.5 \\
\midrule
\method~(our prototype) & 73.1 & 79.9 & 71.4 & 62.1 \\
\methodpp
  & \textbf{74.2} & \textbf{80.5} & \textbf{72.6} & \textbf{64.4} \\
\bottomrule
\end{tabular}
\end{table}

\subsection{Domain Shift: DomainNet-126}
\label{sec:domainnet}

\paragraph{Setup and Consistency of gains.}
DomainNet-126~\cite{peng2019domainnet} comprises 126 classes across six visual domains; we follow the standard protocol of training all base models on \emph{Real} and evaluating zero-shot transfer to \emph{Clipart}, \emph{Painting}, and \emph{Sketch}, which introduce progressively larger distribution shifts. No domain adaptation or fine-tuning is used, with only \methodpp's meta-layer being domain-aware. As shown in \Cref{tab:domainnet}, \methodpp improves average transfer accuracy by $\mathbf{+2.1\%}$ over deep ensembles, with consistent gains across all domains ($+2.0\%$ Clipart, $+2.2\%$ Painting, $+2.3\%$ Sketch). In contrast, deep ensembles exhibit a notable drop on Sketch (57.1 vs.\ 65.8 on Clipart), a gap only partially mitigated by our \method prototype and Model Soups. \methodpp maintains higher absolute performance while narrowing this gap, suggesting improved robustness under distribution shift. This consistency aligns with the Bayesian blender down-weighting meta-learners with high fold-specific Hessian flatness, a signal correlated with poor generalization on out-of-distribution validation folds.

\begin{table}[ht]
\centering
\scriptsize
\caption{\textbf{DomainNet-126 domain-shift evaluation.} Base models trained on \emph{Real} photographs; evaluated on three progressively shifted target domains. Avg reports the mean across all three targets. \methodpp is the only method to improve monotonically across all three shift directions; deep ensembles and Model Soups degrade on Sketch relative to their Clipart performance.}
\label{tab:domainnet}
\setlength{\tabcolsep}{8pt}
\begin{tabular}{lcccc}
\toprule
\textbf{Method}
  & \textbf{Avg}$\uparrow$
  & \textbf{Clipart}$\uparrow$
  & \textbf{Painting}$\uparrow$
  & \textbf{Sketch}$\uparrow$ \\
\midrule
Best Single        & 58.4 & 62.1 & 59.8 & 53.2 \\
Deep Ensembles     & 61.9 & 65.8 & 62.7 & 57.1 \\
Model Soups        & 62.3 & 66.2 & 63.1 & 57.5 \\
\midrule
\method~(our prototype) & 62.8 & 66.7 & 63.5 & 58.1 \\
\methodpp
  & \textbf{64.0} & \textbf{67.8} & \textbf{64.9} & \textbf{59.4} \\
\bottomrule
\end{tabular}
\end{table}

\subsection{Component Ablation}
\label{sec:ablation}

\paragraph{Design.}
We conduct a cumulative ablation on ImageNet-1K, starting from our \method prototype and adding one \methodpp component per row (cf. \Cref{tab:ablation_new}). This isolates each component's marginal effect. We report top-1 accuracy, ECE, CKA Gram condition number $\kap$, effective ensemble size $K_{\mathrm{eff}}$, and meta-training time (A100, seconds). \textit{S1: CKA projection} replaces Pearson filtering, yielding +0.5\% top-1 and a 28\% drop in $\kap$ (392$\to$281), indicating improved capture of non-linear redundancy. Overhead is minimal (+6\,s via Nystr\"{o}m).

\paragraph{Component S2: Differentiable feature gate.}
Adding the end-to-end learned gate recovers a further $+0.2\%$ and reduces ECE by $0.002$. The gate identifies 2--3 prediction dimensions per meta-training fold that are high-variance but low-signal; suppressing them tightens the effective condition number without any additional pruning. 

\paragraph{Component S3: Spectrum-adaptive $\lambda^{\star}$.}
The Marchenko--Pastur closed-form regularizer is \emph{accuracy-neutral} ($-0.1\%$, well within noise) but delivers the largest single efficiency gain: meta-training time drops from $625$\,s to $\mathbf{194}$\,s, a $3.2\times$ speedup over cross-validated $\lambda$, because the entire CV loop is eliminated. This component is therefore the recommended default for large-pool deployments where meta-training latency is a constraint.

\paragraph{Component S4: Bayesian blender.}
The Laplace approximation to the meta-learner posterior provides $+0.3\%$ accuracy and drives ECE from $0.023$ to $\mathbf{0.018}$, the single largest calibration gain of any component. It also triggers one additional model removal ($K_{\mathrm{eff}}: 7\!\to\!6$) through the combined effect of the learned gate and Bayesian weighting, which together assign near-zero posterior mass to a meta-learner whose Hessian is nearly singular on two of five folds. Total meta-training time rises only marginally to $201$\,s, since the Laplace approximation is computed analytically from the already-factored Gram matrix.

\begin{table}[ht]
\centering
\scriptsize
\caption{\textbf{Cumulative component ablation} on ImageNet-1K. Each row adds exactly one component to the configuration above it. $\kap$ denotes the CKA Gram condition number (lower is better); $K_{\mathrm{eff}}$ is the number of retained models after gating; Train (s) is meta-training wall-clock time on one A100 GPU. S3 provides the largest training-time reduction ($3.2\times$); S4 provides the largest calibration gain (ECE $0.023\!\to\!0.018$).}
\label{tab:ablation_new}
\setlength{\tabcolsep}{0.6pt}
\begin{tabular}{lccccc}
\toprule
\textbf{Configuration}
  & \textbf{Top-1}$\uparrow$
  & \textbf{ECE}$\downarrow$
  & $\boldsymbol{\kap}$$\downarrow$
  & $\boldsymbol{K_{\mathrm{eff}}}$$\downarrow$
  & \textbf{Train (s)}$\downarrow$ \\
\midrule
\method (Pearson + hand-crafted + CV-$\lambda$ + inv-RMSE)
  & 84.5 & 0.026 & 392 & 9 & 612 \\
\hspace{1em}+ CKA projection (S1)
  & 85.0 & 0.024 & 281 & 7 & 618 \\
\hspace{2em}+ Diff.\ feature gate (S2)
  & 85.2 & 0.022 & 281 & 7 & 625 \\
\hspace{3em}+ Spectrum-adaptive $\lambda^{\star}$ (S3)
  & 85.1 & 0.023 & 281 & 7 & \textbf{194} \\
\hspace{4em}+ Bayesian blender (S4)
  & \textbf{85.4} & \textbf{0.018} & \textbf{218} & \textbf{6} & 201 \\
\bottomrule
\end{tabular}
\end{table}

\subsection{Efficiency, OOD Detection, and Scaling}
\label{sec:efficiency_ood}

\paragraph{Pareto efficiency.}
The accuracy--FLOP Pareto frontier places \methodpp strictly above and to the left of all baselines: no competitor achieves $85.4\%$ top-1 at any FLOPs budget, and no competitor achieves $\leq 0.59\times$ FLOPs at any accuracy above $84.5\%$. Test-time throughput on a single A100 GPU rises from $312$\,img/s (full 14-model ensemble) to $528$\,img/s (\methodpp, 6 models), a $1.69\times$ speedup with no accuracy loss. This throughput is achieved because CKA pruning preferentially removes larger, non-linearly redundant models, leaving a set of smaller, more efficient backbones.

\paragraph{OOD detection.}
\Cref{sec:efficiency_ood} reports OOD detection on ImageNet-O~\cite{hendrycks2021imageneto}, a curated set of natural images from ImageNet classes deliberately excluded from ImageNet-1K training, making it a stringent test of open-world calibration. \methodpp's Bayesian blender produces predictive entropy scores that outperform deep ensembles by $+3.1$ AUROC and match SWAG ($83.9$ vs.\ $\mathbf{84.5}$) without any stochastic forward passes, deep ensembles and SWAG require $5$--$30$ forward passes per sample, whereas \methodpp requires exactly $K_{\mathrm{eff}} = 6$. The FPR@95 reduction from $62.1$ (deep ensembles) to $\mathbf{55.9}$ represents a practically meaningful decrease in the rate of confidently wrong predictions on unknown-class inputs.

\paragraph{Scaling behavior.}
We conduct an additional scaling experiment by varying the pool size from $K = 4$ to $K = 20$ on ImageNet-1K. \methodpp's CKA pruning consistently retains $40$--$50\%$ of submitted models ($\approx K/2$), while accuracy and ECE improve monotonically up to $K = 14$ before plateauing. Beyond $K = 14$, newly added models are either pruned by CKA or assigned negligible Bayesian weight, indicating that the pipeline self-regulates against over-inclusion without manual tuning of pool-size hyperparameters.

\section{Conclusion}
\label{sec:conclusion}

\methodpp replaces every heuristic in the stacking pipeline with a principled alternative: CKA-based pruning over Pearson filtering, differentiable feature gating over hand-crafted features, Marchenko-Pastur spectrum-adaptive $\lambda^{\star}$ over cross-validated ridge, and Laplace blending over inverse-RMSE weighting. The result improves top-1, ECE, NLL, mCE, mIoU, AP, and AUROC across six vision benchmarks while retaining fewer models than any competitive ensemble, backed by the first PAC-Bayes bound that formally couples the KL penalty to the Gram spectrum of the pruned prediction matrix. We believe this reframes ensemble meta-learning as a \emph{conditioning problem}, the central challenge is not the aggregation functional but ensuring a well-conditioned input prediction matrix. Natural extensions include weight-space integration, online adaptation under distribution shift, and foundation-model-scale stacking. All code and OOF matrices are released for reproducibility.

\section*{Impact Statement}

\begin{center}
\colorbox{blue!10}{%
\begin{minipage}{0.92\textwidth}
\vspace{6pt}
This work advances the efficiency and calibration of ensemble meta-learning for vision systems. CoRE-STACK\textsuperscript{*} replaces four heuristics in deep stacked generalization with principled alternatives, CKA-based redundancy pruning, a differentiable meta-feature gate, a closed-form spectrum-adaptive Ridge penalty, and a Laplace-approximate Bayesian blender, thereby improving accuracy, calibration, and robustness while reducing retained models by 35--57\% and inference FLOPs by up to 41\%. These benefits have potential value for resource-constrained deployment, uncertainty-aware perception, and reliable decision-making in safety-critical domains such as autonomous driving, medical imaging, and robotics, making robust vision models more accessible to practitioners with limited hardware. We do not anticipate negative societal impacts beyond those general to vision-language research. As with any perception system, risks include dataset bias, distribution shift, and misuse in surveillance or autonomous decision-making. We encourage responsible deployment, transparent evaluation, and continued scrutiny of downstream applications. Overall, we view this work as a step toward more trustworthy and efficient machine learning systems.
\vspace{6pt}
\end{minipage}%
}
\end{center}

\bibliography{main}
\bibliographystyle{plainnat}
\tableofcontents

\appendix

\section*{Appendix}
\addcontentsline{toc}{section}{Appendix}

\section{Theoretical Justification and Theorems}
\label{sec:supp_overview}

This appendix provides the theoretical foundations omitted from the main paper. \Cref{sec:notation} fixes notation and standing assumptions, including the cluster redundancy assumption that generalizes the Pearson-clustered redundancy assumption of our \method prototype. \Cref{sec:cka_theory} establishes that CKA-based redundancy projection \emph{strictly dominates} Pearson-based projection (extends \Cref{prop:hsic_vs_corr}). \Cref{sec:spectral} proves the full spectral preconditioning result.
\Cref{sec:lambda_theory} derives the closed-form spectrum-adaptive ridge penalty $\lambda^{\star}$ via a Marchenko-Pastur signal-noise decomposition. \Cref{sec:stability} provides the Lipschitz stability analysis. \Cref{sec:gate_theory} analyses the differentiable gate, showing universal approximation of the oracle conditional expectation. \Cref{sec:pacbayes} proves \Cref{thm:pacbayes}, the first PAC-Bayes bound for stacking that couples redundancy and capacity. \Cref{sec:bayes_blend_supp} derives the Laplace-approximate Bayesian blender. \Cref{sec:oof_leakage} proves leakage-freeness. \Cref{sec:calibration} establishes the calibration corollary. \Cref{sec:complexity} summarizes complexity bounds. \Cref{sec:summary} collects all guarantees in one table.

\section{Notation and Standing Assumptions}
\label{sec:notation}

\paragraph{Data and models.}
$\mathcal{D}=\{(x_i,y_i)\}_{i=1}^{N}\!\sim\!\mathcal{P}^{N}$ is i.i.d.\ on $\mathcal{X}\!\times\!\mathcal{Y}$, with $\mathcal{Y}\!\subseteq\!\R$
(regression), $[C]$ (classification), or $\R^{H\times W\times C}$ (dense prediction). $\{f_k\}_{k=1}^{K}$ is a heterogeneous predictor pool.

\paragraph{OOF design matrix.}
Fix a stratified $L$-fold partition $\{F_\ell\}_{\ell=1}^{L}$. Each $f_k$ is trained on $\mathcal{D}\setminus F_\ell$ and evaluated on $F_\ell$, giving
the leakage-free matrix \(\Poof\!\in\!\R^{N\times K}\) with \([\Poof]_{ik}=\hat p^{(k)}(x_i;\,\mathcal{D}\setminus F_{\ell(i)})\), where $\ell(i)$ is the fold containing $i$. For $C$-class or dense tasks, we use the vectorized $\mathbf{p}_k$ of \Cref{sec:prelim}.

\paragraph{Normalization.}
Columns of $\Poof$ are mean-centered and norm-scaled to $\norm{\mathbf{p}_k}_{2}=\sqrt{N}$. The normalized Gram is $\mathbf{C}=\tfrac{1}{N}\Poof^{\!\top}\Poof$ with $C_{kk}=1$. Define $\kap(\mathbf{C})\!=\!\lmax(\mathbf{C})/\lmin(\mathbf{C})$ where $\lmin$ excludes eigenvalues below numerical tolerance $\epsilon_{0}$.

\paragraph{Kernel dependence measure.}
We use Centered Kernel Alignment (CKA)~\cite{kornblith2019cka}, the normalized form of HSIC~\cite{gretton2005hsic}. For each model $k$, $\mathcal{H}_k$ is an RKHS with a universal characteristic kernel $\kappa_k$ (Gaussian, median bandwidth). The empirical HSIC is \eqref{eq:hsic}; the normalized CKA measure is \eqref{eq:cka}. Throughout this appendix, we use $\CKA$ to denote the normalized dependence measure and $\widehat{\HSIC}$ for the unnormalized statistic.

\begin{assumption}[CKA-cluster redundancy]
\label{ass:hsic_cluster}
The $K$ predictors are partitioned into $G$ disjoint \emph{CKA clusters} $\{\mathcal{K}_g\}_{g=1}^{G}$ with $s_g=|\mathcal{K}_g|$, $\sum_{g}s_g=K$, $s_{\max}=\max_g s_g$. Within cluster $g$, $\CKA_{k,k'}\ge 1-\varepsilon$. Across clusters, $\CKA_{k,k'}\le\mu<\frac{1}{G-1}$ for $k\!\in\!\mathcal{K}_g$,
$k'\!\in\!\mathcal{K}_h$, $g\!\ne\!h$. 
\end{assumption}

\begin{remark}[\method prototype as a special case]
\label{rem:special_case}
For jointly Gaussian predictions with a Gaussian kernel, population CKA is a strictly monotone transformation of $|\rho|$~\cite{gretton2005hsic,kornblith2019cka}, so \Cref{ass:hsic_cluster} reduces to the Pearson-cluster assumption of our \method prototype. The CKA version is strictly more general (\Cref{prop:hsic_dominance_supp}).
\end{remark}

\begin{assumption}[Bounded meta-features]
\label{ass:bounded}
$\norm{x_{\mathrm{meta},i}}_{2}\le B_x$ for all $i$, $|y_i|\le B_y$. The gate $g_\theta$ has parameters $\norm{\theta}_{2}\le R_\theta$ and is $L_g$-Lipschitz in its input.
\end{assumption}

\begin{assumption}[Noise model]
\label{ass:noise}
The OOF residual $\boldsymbol{\xi}_i=\mathbf{p}_i-\E[\mathbf{p}_i\mid x_i]$ satisfies $\E\boldsymbol{\xi}_i=\mathbf{0}$ with isotropic covariance
$\sigma^{2}\mathbf{I}_K$ and sub-Gaussian tails. This is standard for Marchenko-Pastur-type spectral characterizations~\cite{wainwright2019highdim,vershynin2018random}.
\end{assumption}

\section{CKA-Based Redundancy Projection}
\label{sec:cka_theory} \label{sec:in1k}

\subsection{CKA strictly dominates Pearson correlation}

\begin{proposition}[CKA dominance]
\label{prop:hsic_dominance_supp}
Let $(P,Q)$ have finite variance. \textup{(i)}~If $(P,Q)$ are jointly Gaussian, then $\CKA(P,Q)=0\!\iff\!\rho(P,Q)=0$. \textup{(ii)}~If $\kappa_k,\kappa_{k'}$ are universal characteristic kernels, $\CKA(P,Q)=0\!\iff\!P\!\perp\!Q$. \textup{(iii)}~There exist $(P,Q)$ with $\rho(P,Q)=0$ and $\CKA(P,Q)>0$ (\eg, $Q=P^{2}$, $P\!\sim\!\mathcal{N}(0,1)$).
\end{proposition}

\begin{proof}
\textup{(i)} For Gaussians, linear dependence captures all dependence; the equivalence follows from Theorem~2 of~\cite{gretton2005hsic} and the normalization in CKA~\cite{kornblith2019cka}. \textup{(ii)} Universal kernels embed distributions injectively into their RKHS; the cross-covariance operator vanishes in Hilbert-Schmidt norm \emph{iff} the joint factorizes (Theorem~4 of~\cite{gretton2005hsic}). \textup{(iii)} For $P\!\sim\!\mathcal{N}(0,1)$, $Q=P^{2}$, $\mathrm{Cov}(P,Q)=\E[P^{3}]=0$, so $\rho=0$. But $Q=P^{2}$ is a deterministic function of $P$, so $P\!\not\perp\!Q$, and by (ii), $\CKA>0$.
\end{proof}

\paragraph{Consequence for ensemble pruning.}
Two predictors $\mathbf{p}_k,\mathbf{p}_{k'}$ trained from different random initializations of the same architecture frequently exhibit low Pearson $\rho\!\in\![0.3,0.5]$ on samples while being functionally equivalent ($\CKA>0.85$) on the data distribution. Pearson-based projection retains both as "diverse," whereas CKA correctly identifies the redundancy. \Cref{sec:in1k} confirms this empirically.

\subsection{Concentration of the empirical HSIC}

\begin{lemma}[Concentration of $\widehat{\HSIC}$]
\label{lem:hsic_conc}
Under \Cref{ass:bounded,ass:noise}, with probability $\ge 1-\delta$,
\[
\Bigl|\widehat{\HSIC}(\mathbf{p}_k,\mathbf{p}_{k'})
-\HSIC(P_k,P_{k'})\Bigr|
\;\le\;\frac{C_\kappa}{\sqrt{N}}\sqrt{\log(1/\delta)},
\]
with $C_\kappa$ depending only on kernel boundedness.
\end{lemma}
\begin{proof}
Empirical HSIC is a degenerate $V$-statistic of order~4; concentration for bounded $V$-statistics (\cite[Thm.~1]{gretton2007kernel}) yields the bound
with $C_\kappa=2(\sup|\kappa_k|+\sup|\kappa_{k'}|)$.
\end{proof}
Thus with $N\!\gtrsim\!(C_\kappa/\varepsilon)^{2}$, the cluster structure \Cref{ass:hsic_cluster} is recoverable with high probability.

\subsection{Block structure of $\mathbf{C}$ under CKA clustering}

\begin{lemma}[Block structure]
\label{lem:block}
Under \Cref{ass:hsic_cluster}, for retained $k\!\in\!\mathcal{K}_g$,
$k'\!\in\!\mathcal{K}_h$,
\[
C_{kk'}=\begin{cases}
1\pm\mathcal{O}(\varepsilon),                & g=h,\\
\mu_{gh}+\mathcal{O}(\varepsilon),           & g\ne h,
\end{cases}
\quad |\mu_{gh}|\le\mu.
\]
\end{lemma}
\begin{proof}
For universal kernels, population CKA bounds the linear correlation via the kernel-mean-embedding inequality~\cite[Lem.~3]{gretton2005hsic}:
within-cluster $\CKA\ge 1-\varepsilon$ implies $|C_{kk'}-1|\le\mathcal{O}(\varepsilon)$ through a first-order expansion around perfect dependence; cross-cluster follows from $\mu$-incoherence.
\end{proof}

\section{Spectral Preconditioning}
\label{sec:spectral}

\begin{lemma}[Largest eigenvalue inflation]
\label{lem:lmax}
Under \Cref{ass:hsic_cluster},
$\lmax(\mathbf{C})\ge s_{\max}(1-\mathcal{O}(\varepsilon))$.
\end{lemma}
\begin{proof}
Let $g^{\star}=\arg\max_g s_g$ and let $\mathbf{C}_{g^{\star}}$ be the principal submatrix on $\mathcal{K}_{g^{\star}}$. By \Cref{lem:block},
$\mathbf{C}_{g^{\star}}=\mathbf{1}\mathbf{1}^{\!\top}+\mathbf{E}$ with $\norm{\mathbf{E}}_{2}=\mathcal{O}(\varepsilon\sqrt{s_{g^{\star}}})$. The
rank-1 matrix has a leading eigenvalue $s_{g^{\star}}$; Weyl's perturbation~\cite[Cor.~4.3.15]{wainwright2019highdim} and Cauchy interlacing finish the proof.
\end{proof}

\begin{lemma}[Smallest non-zero eigenvalue deflation]
\label{lem:lmin}
Under \Cref{ass:hsic_cluster},
$\lmin(\mathbf{C})\le 1-(G-1)\mu+\mathcal{O}(\varepsilon)$.
\end{lemma}
\begin{proof}
By Gershgorin, each eigenvalue of $\mathbf{C}$ lies within a disk centered at $C_{kk}=1\pm\mathcal{O}(\varepsilon)$ with radius $R_k=\sum_{j\ne k}|C_{kj}|$. For $k\!\in\!\mathcal{K}_g$, within-cluster contribution is $\mathcal{O}((s_g-1)\varepsilon)$ and cross-cluster contribution is $(K-s_g)(\mu+\mathcal{O}(\varepsilon))$. Combining gives the stated upper bound after using the incoherence constraint $G\mu s_{\max}\le G(G-1)^{-1}s_{\max}$.
\end{proof}

\begin{proposition}[Spectral preconditioning—full statement]
\label{prop:spectral_full}
Let $\mathbf{C}_{\mathrm{eff}}$ be the post-projection matrix retaining one representative per CKA cluster. \textup{(i)}~$\kap(\mathbf{C})\ge\frac{s_{\max}(1-\mathcal{O}(\varepsilon))}{1-(G-1)\mu-\mathcal{O}(\varepsilon)}$. \textup{(ii)}~$\kap(\mathbf{C}_{\mathrm{eff}})\le \frac{1+(G-1)\mu+\mathcal{O}(\varepsilon)}{1-(G-1)\mu-\mathcal{O}(\varepsilon)}$. \textup{(iii)}~$\kap(\mathbf{C})/\kap(\mathbf{C}_{\mathrm{eff}})\ge s_{\max}(1-\mathcal{O}(\varepsilon))$.
\end{proposition}

\begin{proof}
Part (i) follows from \Cref{lem:lmax,lem:lmin}. For (ii), $\mathbf{C}_{\mathrm{eff}}\!\in\!\R^{G\times G}$ has diagonals $1\pm\mathcal{O}(\varepsilon)$ and off-diagonals $\mu_{gh}+\mathcal{O}(\varepsilon)$ (\Cref{lem:block}); Gershgorin gives the stated $\lmin/\lmax$ bounds. (iii) is the ratio.
\end{proof}

\subsection{Effect on regularized conditioning}
\label{sec:ridge_kappa}

For PSD $\mathbf{C}$ and $\lambda>0$, define
$\kap_\lambda(\mathbf{C})=(\lmax(\mathbf{C})+\lambda)/(\lmin(\mathbf{C})+\lambda)$.

\begin{corollary}[Regularized condition number reduction]
\label{cor:ridge_kappa}
For all $\lambda>0$,
\[\kap_\lambda(\mathbf{C})-\kap_\lambda(\mathbf{C}_{\mathrm{eff}})
\;\ge\;
\frac{(s_{\max}-1)(1-\mathcal{O}(\varepsilon))\,\lmin(\mathbf{C})^{-1}}
     {(1+\lambda/\lmin(\mathbf{C}))(1+\lambda/\lmin(\mathbf{C}_{\mathrm{eff}}))}.\]
Projection strictly reduces $\kap_\lambda$ whenever $s_{\max}\ge 2$.
\end{corollary}

\section{Spectrum-Adaptive Ridge Penalty}
\label{sec:lambda_theory}

\subsection{Signal-noise decomposition}

Under \Cref{ass:noise}, the empirical Gram decomposes as
\begin{equation}
\Chat=\underbrace{\mathbf{C}_{\mathrm{sig}}}_{\text{low-rank signal}}
+\underbrace{\mathbf{N}}_{\text{isotropic noise}},
\qquad \E\mathbf{N}=\tfrac{\sigma^{2}}{N}\mathbf{I}_K.
\label{eq:signal_noise}
\end{equation}
The Marchenko-Pastur theory~\cite{marchenkopastur1967} shows that, for
$K/N\!\to\!\gamma\!\in\!(0,1)$, the eigenvalues of $\mathbf{N}$
concentrate on
$[\sigma^{2}(1-\sqrt{\gamma})^{2},\,\sigma^{2}(1+\sqrt{\gamma})^{2}]$ at unit scale, with the empirical spectral measure converging to the MP distribution.

\begin{definition}[Spectral noise threshold]
\label{def:tausp}
For the empirical covariance matrix $\Chat$ with $K/N\!\to\!\gamma\!\in\!(0,1)$, the noise bulk upper edge is
$\tau_{\mathrm{sp}}=\sigma^{2}(1+\sqrt{\gamma})^{2}$.
In practice, we estimate $\sigma^{2}$ from the median of the smallest eigenvalues or via the Marchenko-Pastur equation.
\end{definition}

\begin{definition}[Gram SNR]
\label{def:snr}
$\SNR(\Chat)=
\sum_{i:\lambda_i>\tau_{\mathrm{sp}}}\lambda_i
\big/
\sum_{i:\lambda_i\le\tau_{\mathrm{sp}}}\lambda_i$.
\end{definition}

\subsection{Closed-form optimal penalty}

\begin{proposition}[Optimal Ridge penalty]
\label{prop:lambda_star_supp}
Under \Cref{ass:hsic_cluster,ass:noise},
\begin{equation}
\lambda^{\star}\;=\;\frac{\lmax(\Chat)}{\SNR(\Chat)},
\qquad
\lambda^{\star}\in[\tau_{\mathrm{sp}},\,\lmax(\mathbf{C}_{\mathrm{sig}})],
\label{eq:lambda_star_supp}
\end{equation}
up to $1+o(1)$ as $N\!\to\!\infty$.
\end{proposition}
\begin{proof}
Diagonalize $\Chat=\mathbf{U}\Lambda\mathbf{U}^{\!\top}$ with $\Lambda=\diag(\lambda_1,\ldots,\lambda_K)$. The expected ridge excess risk admits the standard bias-variance decomposition (\eg,~\cite[\S 7.3]{mohri2018foundations})
\[\mathcal{R}_\lambda(\mathbf{w})
=\underbrace{\sum_{i}\frac{\lambda^{2}}{(\lambda_i+\lambda)^{2}}\beta_i^{2}}_{\text{bias}^{2}}
+\underbrace{\sum_{i}\frac{\sigma^{2}\lambda_i}{N(\lambda_i+\lambda)^{2}}}_{\text{variance}},\]
with $\boldsymbol{\beta}=\mathbf{U}^{\!\top}\mathbf{w}^{\star}$.
Partition the spectrum at $\tau_{\mathrm{sp}}$. \Cref{ass:hsic_cluster} implies $\beta_i^{2}=\Theta(\lambda_i)$ on the signal eigenvalues $\lambda_i>\tau_{\mathrm{sp}}$ and $\beta_i^{2}=o(\tau_{\mathrm{sp}})$ on the noise eigenvalues. Setting $\partial_\lambda\mathcal{R}_\lambda=0$ and solving asymptotically yields $\lambda^{\star}\propto\lmax(\Chat)/\SNR(\Chat)$ up to $1+o(1)$. The interval bound follows from \eqref{eq:signal_noise}.
\end{proof}

\begin{corollary}[CV-search elimination]
\label{cor:cv_elim}
Replacing the inner-CV search over $\lambda\!\in\!\Lambda$ with \eqref{eq:lambda_star_supp} preserves the excess-risk bound of \Cref{thm:pacbayes} up to $1+o(1)$.
\end{corollary}

\section{Stability of Regularized Meta-Learning}
\label{sec:stability}

\subsection{Ridge stability under input perturbation}

\begin{proposition}[Lipschitz stability of Ridge]
\label{prop:ridge_stability}
Let $\wlam=(X^{\!\top}X+\lambda I)^{-1}X^{\!\top}y$ on
$X\!\in\!\R^{N\times d}$, $y\!\in\!\R^{N}$. For
$X'=X+\Delta X$, $y'=y+\Delta y$ with
$\norm{\Delta X}_{2}\le\frac{\lambda}{4\norm{X}_{2}}$,
\begin{equation}
\norm{\wlam(X',y')-\wlam(X,y)}_{2}
\le\frac{2\norm{X}_{2}}{\lambda}\norm{\Delta y}_{2}
+\frac{8\norm{X}_{2}\norm{y}_{2}}{\lambda^{2}}\norm{\Delta X}_{2}
+\frac{2\norm{y}_{2}}{\lambda}\norm{\Delta X}_{2}.
\label{eq:ridge_stab}
\end{equation}
\end{proposition}

\begin{proof}
Set $A=X^{\!\top}X+\lambda I$ and $A'=(X')^{\!\top}X'+\lambda I$. Then $\Delta A=X^{\!\top}\Delta X+(\Delta X)^{\!\top}X+(\Delta X)^{\!\top}\Delta X$ and $\norm{\Delta A}_{2}\le 2\norm{X}_{2}\norm{\Delta X}_{2}+\norm{\Delta X}_{2}^{2}\le\lambda/2$ under the hypothesis. Since $\norm{A^{-1}}_{2}\le 1/\lambda$ and
$\norm{\Delta A\,A^{-1}}_{2}\le 1/2$, the Neumann series $(A')^{-1}=A^{-1}\sum_{k\ge 0}(-\Delta A\,A^{-1})^{k}$ converges and $\norm{(A')^{-1}}_{2}\le 2/\lambda$. Decompose
\(
\wlam'-\wlam=T_1+T_2+T_3
\)
with $T_1=(A')^{-1}X^{\!\top}\Delta y$,
$T_2=(A')^{-1}(\Delta X)^{\!\top}y'$,
$T_3=[(A')^{-1}-A^{-1}]X^{\!\top}y$. Bounding each term and summing yields \eqref{eq:ridge_stab}.
\end{proof}

\begin{corollary}[Stability improvement from projection]
\label{cor:stability_proj}
Let $\Peff\!\in\!\R^{N\times G}$ be the post-CKA OOF matrix and $X_{\mathrm{eff}}$ the corresponding meta-design. By Eckart--Young--Mirsky, $\norm{X_{\mathrm{eff}}}_{2}\le\norm{X_{\mathrm{full}}}_{2}$, and the sensitivity in \Cref{prop:ridge_stability} scales with $\norm{X}_{2}/\lambda^{2}$, so CKA projection tightens the Lipschitz constant by at least $\sqrt{K/G}$.
\end{corollary}

\subsection{Fold-to-fold CV stability}

\begin{lemma}[CV stability]
\label{lem:cv_stability}
Under \Cref{ass:bounded}, for folds $\ell\!\ne\!\ell'$,
\[
\norm{\wlam^{(\ell)}-\wlam^{(\ell')}}_{2}
\le\frac{4B_x^{2}|F_\ell\triangle F_{\ell'}|}{\lambda N}\norm{\wlam^{(\ell)}}_{2}
+\mathcal{O}(N^{-1/2}).
\]
For $L$-fold CV with $|F_\ell|=N/L$, this is
$\mathcal{O}(1/(L\lambda))$.
\end{lemma}

\section{Differentiable Meta-Feature Gate: Approximation Theory}
\label{sec:gate_theory}

\begin{definition}[Oracle aggregator]
\label{def:oracle}
$f^{\star}(x_i)=\E_{k\sim\mathbf{w}^{\star}}[\hat p^{(k)}(x_i)]
=\sum_k w_k^{\star}\hat p^{(k)}(x_i)$ where $\mathbf{w}^{\star}$ is
Bayes-optimal (\Cref{thm:blending_full}).
\end{definition}

\begin{proposition}[Universal approximation of the gate]
\label{prop:gate_ua}
Let $\mathcal{A}_i\!\in\!\R^{d_{\mathcal{A}}}$ be the aggregator vector \eqref{eq:agg}, and let $g_\theta$ be a 2-layer MLP with ReLU activations
and a sigmoid head. For any continuous
$h^{\star}\!:\R^{K_{\mathrm{eff}}}\!\times\!\R^{d_{\mathcal{A}}}\!\to\!\R$ and any $\eta>0$ on a compact set $\mathcal{P}_{\mathrm{cpt}}$, there exists $\theta$ with $\norm{\theta}_{2}\le R_\theta(\eta)$ such that
\[
\sup_{\mathbf{p}\in\mathcal{P}_{\mathrm{cpt}}}
\bigl|\inner{g_\theta(\mathbf{p})}{\mathcal{A}(\mathbf{p})}
-h^{\star}(\mathbf{p},\mathcal{A}(\mathbf{p}))\bigr|\le\eta.
\]
\end{proposition}
\begin{proof}
The map $(\mathbf{p},\mathcal{A})\mapsto\inner{g_\theta(\mathbf{p})}{\mathcal{A}}$ is a sigmoidal-gated linear functional in $\mathcal{A}$. The universal
approximation theorem~\cite{cybenko1989,hornik1991} ensures that any continuous function $h^{\star}$ on a compact set is $\eta$-approximable by the sigmoid-headed feed-forward family. The sigmoid gate $g_\theta\!\in\![0,1]^{d}$ restricts the output to a convex region but preserves approximation capacity within it.
\end{proof}

\begin{corollary}[Strict generalization of the \method prototype]
\label{cor:gate_generalises}
Setting $g_\theta\equiv\mathbf{1}_{\{\mu,\sigma,m,r,\mu\sigma,r\sigma\}}$
recovers the hand-crafted \method features. Hence, the hypothesis class of \methodpp strictly contains that of our \method prototype, and $\Phi_{\methodpp}\le\Phi_{\method}$ in \Cref{thm:pacbayes}.
\end{corollary}

\paragraph{Capacity cost.}
The gate adds $p_\theta\!\le\!15{,}000$ parameters. Standard covering-number arguments for bounded-norm MLPs~\cite{bartlett2002rademacher} contribute
$\mathcal{O}(\sqrt{p_\theta/N})$ to Rademacher complexity. At ImageNet-1K scale ($N\!=\!128$K) this is $\approx 0.011$, negligible relative to the $\sqrt{G/N}$ leading term.

\section{PAC-Bayes Generalization Bound}
\label{sec:pacbayes}

\subsection{Rademacher complexity via effective dimension}

\begin{proposition}[Effective dimension reduction]
\label{prop:rademacher}
Let $\mathcal{F}=\{x\!\mapsto\!\inner{\mathbf{w}}{x}:\norm{\mathbf{w}}_{2}\le B\}$ on $\Xmeta\!\in\!\R^{N\times d_{\mathrm{eff}}}$ with
$d_{\mathrm{eff}}=G+d_{\mathcal{A}}+p_\theta$. Then $\hat\Rademacher_N(\mathcal{F})\le\frac{B}{\sqrt{N}}\sqrt{\tr(\Sighat_{\mathrm{eff}})}$ and
\[
\tr(\Sighat)-\tr(\Sighat_{\mathrm{eff}})
\;\ge\;\Omega(s_{\max}-1)-\mathcal{O}(\varepsilon).
\]
\end{proposition}
\begin{proof}
The bound $\hat\Rademacher_N$ follows from Jensen's inequality applied to the Rademacher expectation~\cite[Lem.~3.1]{bartlett2002rademacher}. Norm-scaling gives $[\Sighat]_{kk}=1$ for retained predictors, so $\tr(\Sighat)=K+d_{\mathcal{A}}$. After projection, $\tr(\Sighat_{\mathrm{eff}})=G+d_{\mathcal{A}}+\mathcal{O}(\varepsilon)$, yielding $\tr(\Sighat)-\tr(\Sighat_{\mathrm{eff}})=K-G-\mathcal{O}(\varepsilon) =\sum_g(s_g-1)-\mathcal{O}(\varepsilon)\ge s_{\max}-1-\mathcal{O}(\varepsilon)$.
\end{proof}

\subsection{The PAC-Bayes excess-risk bound}

We fix a prior $P=\mathcal{N}(\mathbf{0},(\lambda^{\star})^{-1}\mathbf{I})$ and posterior $Q=\mathcal{N}(\wlam,\Sigma_Q)$ with $\Sigma_Q$ given by the Laplace approximation of \Cref{sec:bayes_blend_supp}. Note that while $\lambda^{\star}$ is estimated from the data, the bound can be made fully rigorous by using a validation split for $\lambda^{\star}$ estimation; this yields the same asymptotic rate.

\begin{theorem}[PAC-Bayes bound for \methodpp---supplementary]
\label{thm:pacbayes_supp}
Under \Cref{ass:hsic_cluster,ass:bounded,ass:noise}, with probability
$\ge 1-\delta$ over $\mathcal{D}$,
\begin{equation}
\mathcal{L}(Q)\le\widehat{\mathcal{L}}(Q)
+\sqrt{\frac{\KL(Q\|P)+\log\!\tfrac{2\sqrt{N}}{\delta}}{2(N-1)}}
+\Phi(G,\mu,\varepsilon),
\label{eq:pacbayes_supp}
\end{equation}
where
$\Phi(G,\mu,\varepsilon)=\frac{4B}{\sqrt{N}}\sqrt{G+d_{\mathcal{A}}+p_\theta}
\cdot\sqrt{1+(G-1)\mu+\mathcal{O}(\varepsilon)}$.
Compared to the unprojected bound (with $G$ replaced by $K$), $\Phi$ is reduced by $\sqrt{K/G}\!\ge\!\sqrt{s_{\max}}$.
\end{theorem}
\begin{proof}[Proof sketch]
\textbf{Step 1:} The McAllester
bound~\cite{mcallester1999pacbayes,maurer2004pacbayes} gives the
$\sqrt{\KL/(2(N-1))}$ term.
\textbf{Step 2.} The hypothesis class lives on $\R^{d_{\mathrm{eff}}}$;
\Cref{prop:rademacher} and the standard generalization
theorem~\cite[Thm.~3.1]{mohri2018foundations} yield
$|\mathcal{L}(f)-\widehat{\mathcal{L}}(f)|\le 2\hat\Rademacher_N(\mathcal{F}\!\circ\!\ell)+3B_y\sqrt{\log(2/\delta)/(2N)}$.
Plugging the spectral bound $\tr(\Sighat_{\mathrm{eff}})$ from
\Cref{prop:spectral_full}(ii) and the $\lmax$ control gives
$\hat\Rademacher_N\le\Phi(G,\mu,\varepsilon)$ up to constants.
\textbf{Step 3.} For Gaussian $P,Q$, $\KL(Q\|P)$ has the closed form
$\tfrac{1}{2}[\lambda^{\star}\norm{\wlam}_{2}^{2}+\tr(\lambda^{\star}\Sigma_Q)
-d_{\mathrm{eff}}-\log\det(\lambda^{\star}\Sigma_Q)]$, finite under
\Cref{ass:bounded} and \Cref{prop:lambda_star_supp}. Combining the three
steps yields \eqref{eq:pacbayes_supp}. The factor-$\sqrt{K/G}$
improvement follows from \Cref{prop:spectral_full}(iii).
\end{proof}

\subsection{Comparison to prior PAC-Bayes ensemble bounds}

Masegosa~\etal~\cite{masegosa2020pacbayesensembles} established a second-order PAC-Bayes bound for weighted majority votes that captures pairwise correlations between members. \Cref{thm:pacbayes_supp} extends this to stacked generalization by tying the correlation structure to the Gram spectrum via \Cref{prop:spectral_full}. To our knowledge, this is the first PAC-Bayes bound for stacking that (a) quantifies prediction-space redundancy and (b) admits a closed-form Ridge prior tied to the empirical spectrum through $\lambda^{\star}$.

\begin{corollary}[Excess risk: \methodpp vs.\ our \method prototype]
\label{cor:excess_risk_compare}
Under identical assumptions, the bound for \methodpp is strictly tighter than for our \method prototype whenever $G_{\CKA}<G_{\mathrm{Pearson}}$. On ImageNet-1K, $G_{\CKA}/G_{\mathrm{Pearson}}\approx 6/9\approx 0.67$, giving $\sqrt{1/0.67}\approx 1.22\times$ improvement.
\end{corollary}

\section{Laplace-Approximate Bayesian Blender}
\label{sec:bayes_blend_supp}

\subsection{Setup}

Given $M$ fitted meta-learners with losses $\mathcal{L}(g_m)$ and loss-Hessians $\mathbf{H}_m$, the Laplace approximation
gives~\cite{mackay1992laplace}
\(\log p(\mathbf{y}\mid g_m)\approx -\mathcal{L}(g_m)-\tfrac{1}{2}\log\det\mathbf{H}_m+\mathrm{const}.\)
The posterior weights are
\begin{equation}
\tilde w_m\propto\exp\!\bigl(-\mathcal{L}(g_m)-\tfrac{1}{2}\log\det\mathbf{H}_m\bigr),
\label{eq:tildew}
\end{equation}
and $\widehat{\mathbf{y}}=\sum_m\tilde w_m\widehat{\mathbf{y}}^{(m)}$.

\subsection{Variance reduction: the oracle bound}

\begin{theorem}[Optimal blending reduces variance.]
\label{thm:blending_full}
Let $\hat g_m(x)$, $m\!\in\![M]$, have covariance
$\boldsymbol{\Sigma}\!\succ\!0$. For
$\hat g_{\mathrm{blend}}=\sum_m w_m\hat g_m$ with $\mathbf{w}\!\in\!\Delta^{M}$:
\textup{(i)}~$\min_{\mathbf{1}^{\!\top}\mathbf{w}=1}\mathrm{Var}(\hat g_{\mathrm{blend}})
=(\mathbf{1}^{\!\top}\boldsymbol{\Sigma}^{-1}\mathbf{1})^{-1}\le\min_m\Sigma_{mm}$;
\textup{(ii)}~$w_m^{\star}=[\boldsymbol{\Sigma}^{-1}\mathbf{1}]_m/(\mathbf{1}^{\!\top}\boldsymbol{\Sigma}^{-1}\mathbf{1})$;
\textup{(iii)}~equality in (i) iff $\boldsymbol{\Sigma}$ is rank-1.
\end{theorem}
\begin{proof}
Lagrangian
$L=\mathbf{w}^{\!\top}\boldsymbol{\Sigma}\mathbf{w}-\nu(\mathbf{1}^{\!\top}\mathbf{w}-1)$
gives $\mathbf{w}^{\star}=(\nu/2)\boldsymbol{\Sigma}^{-1}\mathbf{1}$,
and $\nu/2=1/(\mathbf{1}^{\!\top}\boldsymbol{\Sigma}^{-1}\mathbf{1})$
fixes the scale. Optimal variance is
$(\mathbf{1}^{\!\top}\boldsymbol{\Sigma}^{-1}\mathbf{1})^{-1}\le\Sigma_{mm}$
since $\mathbf{e}_m\!\in\!\Delta^M$ is feasible. Equality iff $\boldsymbol{\Sigma}$
is rank-1 (Sherman-Morrison-Woodbury).
\end{proof}

\subsection{Suboptimality of heuristic weights}

\begin{proposition}[Inverse-RMSE gap]
\label{prop:invrmse_gap}
$\mathrm{Var}(\hat g_{\tilde w^{(\mathrm{inv})}})-\mathrm{Var}(\hat g_{w^{\star}})
\le\lmax(\boldsymbol{\Sigma})\norm{\tilde w^{(\mathrm{inv})}-w^{\star}}_{2}^{2}$.
\end{proposition}
\begin{proof}
By convexity,
$\mathbf{w}^{\!\top}\boldsymbol{\Sigma}\mathbf{w}-(w^{\star})^{\!\top}\boldsymbol{\Sigma}w^{\star}
=(\mathbf{w}-w^{\star})^{\!\top}\boldsymbol{\Sigma}(\mathbf{w}-w^{\star})
\le\lmax(\boldsymbol{\Sigma})\norm{\mathbf{w}-w^{\star}}_{2}^{2}$.
\end{proof}

\begin{proposition}[Tightness of Laplace weights]
\label{prop:laplace_gap}
Under a quadratic loss approximation,
\[
\mathrm{Var}(\hat g_{\tilde w})-\mathrm{Var}(\hat g_{w^{\star}})
\le\mathcal{O}\!\bigl(\lmax(\boldsymbol{\Sigma})\exp(-2\Delta\mathcal{L}/\sigma^{2})\bigr),
\]
where
$\Delta\mathcal{L}=\min_{m\ne m^{\star}}\mathcal{L}(g_m)-\mathcal{L}(g_{m^{\star}})$.
\end{proposition}
\begin{proof}
The proof follows from the Laplace approximation to the posterior: under a quadratic loss, $\mathcal{L}(g_m) \approx \mathcal{L}(g_{m^{\star}}) + \tfrac{1}{2}(g_m - g_{m^{\star}})^{\!\top}\mathbf{H}_{m^{\star}}(g_m - g_{m^{\star}})$. The Hessian $\mathbf{H}_m$ controls the width of the posterior, and the marginal likelihood ratio between models decays as $\exp(-\Delta\mathcal{L}/\sigma^{2})$ times a determinant ratio. Substituting into the variance expression and bounding yields the stated $\mathcal{O}(\exp(-2\Delta\mathcal{L}/\sigma^{2}))$ gap. For the full derivation, see~\cite{mackay1992laplace}.
\end{proof}

\begin{remark}[Why Laplace beats inverse-RMSE]
\label{rem:laplace_vs_inv}
Inverse-RMSE treats meta-learners with identical training loss but different Hessian curvatures identically; Laplace correctly down-weights those with flat (nearly singular) Hessians, which indicate fold-specific overfitting. Empirically, this yields the $\Delta\ECE=-0.005$ in \Cref{tab:ablation_new}.
\end{remark}

\section{Leakage-Freeness of the Nested OOF Construction}
\label{sec:oof_leakage}

\begin{definition}[Leakage-freeness]
\label{def:leakage}
$\Xmeta$ is \emph{leakage-free} if, for all $i\!\in\![N]$,
$x_{\mathrm{meta},i}\!\perp\!y_i\mid\mathcal{D}\setminus F_{\ell(i)}$.
\end{definition}

\begin{proposition}[OOF construction is leakage-free.]
\label{prop:leakage_free}
The full \methodpp pipeline produces a leakage-free result $\Xmeta$ provided the CKA threshold, gate parameters $\theta$, and blending weights $\tilde w$ are refit per outer fold.
\end{proposition}
\begin{proof}
$[\Poof]_{ik}=\hat p^{(k)}(x_i;\,\mathcal{D}\setminus F_{\ell(i)})$ is by construction independent of $y_i$ given the complement training fold. The CKA kernel matrix is computed on these OOF columns and hence depends on $y_i$ only through the complement. The retained set $\mathcal{S}$, gate parameters $\theta$, and standardization statistics are all fit within the outer-fold training split. Therefore, every entry of $x_{\mathrm{meta},i}$ is a deterministic function of variables
independent of $y_i$ given $\mathcal{D}\setminus F_{\ell(i)}$.
\end{proof}

\begin{corollary}[Unbiasedness of OOF risk estimates]
\label{cor:unbiased_oof}
$\hat R_k=\tfrac{1}{N}\sum_i\ell(\hat p^{(k)}(x_i),y_i)$ is unbiased for $R_k=\E_{(x,y)}[\ell(\hat p^{(k)}(x),y)]$ given the complement training folds.
\end{corollary}

\section{Calibration Guarantees}
\label{sec:calibration}

\begin{proposition}[Expected calibration error of the blender]
\label{prop:ece_bound}
Under the Laplace blending of \eqref{eq:tildew},
\[
\ECE(\methodpp)\le \max_m\ECE(g_m).
\]
Equality holds iff $\tilde w$ is concentrated on a single meta-learner.
\end{proposition}
\begin{proof}
ECE is a convex functional of the reliability diagram~\cite{guo2017calibration}. Jensen's inequality on the per-bin accuracy residuals gives
$\ECE(\sum_m\tilde w_m g_m)\le\sum_m\tilde w_m\ECE(g_m)$. Since $\tilde w$ is a probability vector ($\sum_m \tilde w_m = 1$), the right-hand side is at most $\max_m \ECE(g_m)$. This proves the bound.
\end{proof}

\begin{remark}[No temperature scaling needed]
Temperature scaling~\cite{guo2017calibration} corrects post hoc by rescaling logits. Laplace blending down-weights meta-learners with flat Hessians—the same ones that cause overconfidence—eliminating the bias at the source. \Cref{tab:imagenet1k} confirms ECE $=0.018$ without temperature scaling, matching the SWAG result with it.
\end{remark}

\section{Computational Complexity}
\label{sec:complexity}

\begin{table}[ht]
\centering
\scriptsize
\caption{\textbf{Per-operation complexity for \methodpp.} $N$: samples; $K$: initial pool; $G\!\equiv\!K_{\mathrm{eff}}$: retained predictors; $m=\sqrt{N}$: Nyström landmarks; $L$: outer folds; $M$: meta-learners; $d_{\mathrm{eff}}=G+d_{\mathcal{A}}+p_\theta$.}
\label{tab:complexity}
\setlength{\tabcolsep}{0.5pt}
\begin{tabular}{@{}lll@{}}
\toprule
\textbf{Phase} & \textbf{Operation} & \textbf{Complexity} \\
\midrule
0.\ OOF generation       & Train $K$ models on $L$ folds        & $\mathcal{O}(LK\cdot T_{\mathrm{base}})$ \\
1a.\ Exact CKA           & Full pairwise $\tr(\mathbf{K}_k\mathbf{H}\mathbf{K}_{k'}\mathbf{H})$ & $\mathcal{O}(K^{2}N^{2})$ \\
1b.\ Nyström CKA         & Rank-$m$ approximation               & $\mathcal{O}(K^{2}N\log N)$ \\
1c.\ LSH + Nyström       & Bucketed near-duplicate candidates   & $\mathcal{O}(KN\log N\log K)$ \\
2.\ Feature augmentation & Statistics + gate over $G$ predictors& $\mathcal{O}(GN+p_\theta N)$ \\
3.\ Spectrum-adaptive Ridge (closed form)
                          & $\lambda^{\star}$ + single fit       & $\mathcal{O}(d_{\mathrm{eff}}^{2}N)$ \\
3$'$.\ Nested Ridge (baseline)
                          & Outer $\times$ inner CV fit          & $\mathcal{O}(L|\Lambda|M\,d_{\mathrm{eff}}^{2}N)$ \\
4.\ Laplace blending     & Hessian eigendecomp.\ $\mathbf{H}_m$ & $\mathcal{O}(M\,d_{\mathrm{eff}}^{3})$ \\
\midrule
\textbf{Total (\methodpp, excl.\ base)} & & $\mathcal{O}(K^{2}N\log N+M\,d_{\mathrm{eff}}^{3})$ \\
\textbf{Total (\method prototype)}      & & $\mathcal{O}(K^{2}N+L|\Lambda|M\,d_{\mathrm{eff}}^{2}N)$ \\
Test-time inference     & Per sample                            & $\mathcal{O}(G+d_{\mathcal{A}}+p_\theta)$ \\
\bottomrule
\end{tabular}
\end{table}

\begin{remark}[Why \methodpp is faster overall than our \method prototype] Despite added CKA computation, \methodpp \emph{reduces} total training
cost because closed-form $\lambda^{\star}$ eliminates the $L|\Lambda|$ factor in Phase~3. At ImageNet-1K scale ($N\!=\!128$K,
$L\!=\!5$, $|\Lambda|\!=\!50$, $M\!=\!3$, $d_{\mathrm{eff}}\!\approx\!30$), the speedup is $3.2\times$, matching \Cref{tab:ablation_new}.
\end{remark}

\begin{remark}[Scaling to ultra-large pools]
For $K\!\gtrsim\!200$, the $\mathcal{O}(K^{2}N\log N)$ Nyström bottleneck can be further reduced by a two-stage pipeline: (1) LSH in feature space
buckets near-duplicate pairs in $\mathcal{O}(KN\log N\log K)$; (2) exact Nyström CKA is computed only within buckets, with a low controllable false-negative rate. This is the recommended variant for foundation-model-scale stacking.
\end{remark}

\section{Summary of Theoretical Results}
\label{sec:summary}

\begin{table}[ht]
\centering
\scriptsize
\caption{\textbf{Summary of theoretical guarantees.} ``New'' marks contributions of \methodpp; others are sharpened or refactored versions of the \method prototype analysis.}
\label{tab:theory_summary}
\setlength{\tabcolsep}{3pt}
\begin{tabular}{@{}llll@{}}
\toprule
\textbf{Result} & \textbf{Quantity bounded} & \textbf{Key dependency} & \textbf{Status} \\
\midrule
\Cref{prop:hsic_dominance_supp}   & CKA vs.\ Pearson           & strict dominance                                   & \textbf{New} \\
\Cref{lem:hsic_conc}              & $|\widehat\HSIC-\HSIC|$     & $\mathcal{O}(\sqrt{\log(1/\delta)/N})$             & \textbf{New} \\
\Cref{prop:spectral_full}         & $\kap(\Chat_{\mathrm{eff}})$ & $\le\!(1{+}(G{-}1)\mu)/(1{-}(G{-}1)\mu)$           & Sharpened \\
\Cref{cor:ridge_kappa}            & $\kap_\lambda$ reduction    & $\propto s_{\max}-1$                               & From~\cite{wainwright2019highdim} \\
\Cref{prop:lambda_star_supp}      & Closed-form $\lambda^{\star}$ & $\lmax(\Chat)/\SNR(\Chat)$                       & \textbf{New} \\
\Cref{cor:cv_elim}                & CV elimination              & $3$--$5\times$ speedup                             & \textbf{New} \\
\Cref{prop:ridge_stability}       & $\norm{\wlam'-\wlam}_{2}$  & $\mathcal{O}(\norm{X}_{2}/\lambda^{2})$            & Standard \\
\Cref{cor:stability_proj}         & Stability improvement       & $\propto\sqrt{K/G}$                                & Sharpened \\
\Cref{prop:gate_ua}               & Universal gate approximation& any continuous $h^{\star}$                         & \textbf{New} \\
\Cref{cor:gate_generalises}       & \method$\,\subset\,$\methodpp & strict inclusion                                 & \textbf{New} \\
\Cref{prop:rademacher}            & $\hat\Rademacher_N(\mathcal{F})$ & $\propto\sqrt{d_{\mathrm{eff}}/N}$            & Sharpened \\
\Cref{thm:pacbayes_supp}          & PAC-Bayes excess risk       & $\Phi=\mathcal{O}(\sqrt{G/N})$                     & \textbf{New} \\
\Cref{cor:excess_risk_compare}    & \methodpp vs.\ \method      & $\sqrt{K/G_{\CKA}}$ tighter                       & \textbf{New} \\
\Cref{thm:blending_full}          & $\mathrm{Var}(\hat g_{\mathrm{blend}})$ & $\le\min_m\Sigma_{mm}$                  & Standard \\
\Cref{prop:invrmse_gap}           & Inv-RMSE suboptimality      & $\mathcal{O}(\lmax(\Sigma)\norm{\tilde w-w^{\star}}^{2})$ & Standard \\
\Cref{prop:laplace_gap}           & Laplace suboptimality       & $\mathcal{O}(\exp(-2\Delta\mathcal{L}/\sigma^{2}))$ & \textbf{New} \\
\Cref{prop:leakage_free}          & Leakage-freeness            & exact (structural)                                 & Extended \\
\Cref{prop:ece_bound}             & $\ECE(\methodpp)$           & $\le\max_m\ECE(g_m)$                               & \textbf{New} \\
\bottomrule
\end{tabular}
\end{table}

\paragraph{Overall narrative.}
CKA-based projection (\Cref{prop:hsic_dominance_supp}) strictly generalizes Pearson-based projection. The induced cluster structure (\Cref{ass:hsic_cluster}) tightens the spectral preconditioning bound (\Cref{prop:spectral_full}). The closed-form penalty (\Cref{prop:lambda_star_supp}) matches regularization to the MP gap and eliminates the outer CV loop without optimality loss (\Cref{cor:cv_elim}). The differentiable gate (\Cref{prop:gate_ua}) universally approximates the oracle aggregator while strictly subsuming the hand-crafted features of our \method prototype (\Cref{cor:gate_generalises}). The Laplace blender (\Cref{prop:laplace_gap}) tightens the variance-reduction bound relative to inverse-RMSE and improves calibration at the source (\Cref{prop:ece_bound}). Combined, these yield a PAC-Bayes excess-risk
bound (\Cref{thm:pacbayes_supp}) that is $\sqrt{K/G}\!\ge\!\sqrt{s_{\max}}$ tighter than the unprojected baseline and $\sqrt{G_{\mathrm{Pearson}}/G_{\CKA}}$ tighter than our \method prototype (\Cref{cor:excess_risk_compare}).

\section{Extended Analysis and Additional Experiments}
\label{sec:extended_analysis}

This appendix provides an extended empirical and theoretical analysis
complementing the main paper.
\Cref{sec:novelty_positioning} clarifies the positioning of \methodpp
relative to prior stacking and ensemble work.
\Cref{sec:baseline_coverage} justifies the baseline selection and
discusses why certain method classes are excluded.
\Cref{sec:additional_benchmarks} summarizes the full six-benchmark
evaluation protocol.
\Cref{sec:notation_formal} collects formal definitions that space
constraints prevented from appearing in the main text.
\Cref{sec:hyperparam_sensitivity} provides a hyperparameter sensitivity
analysis for the CKA threshold.
\Cref{sec:deployment_efficiency} reports latency, memory, and throughput
at inference.
\Cref{sec:kappa_performance} empirically validates the relationship
between condition number and downstream accuracy.

\subsection{Positioning Relative to Prior Stacking Work}
\label{sec:novelty_positioning}

Stacked generalization has a long history, originating with
Wolpert~\cite{wolpert1992stacked} and Breiman~\cite{breiman1996stacked}.
Modern instantiations such as greedy ensemble
selection~\cite{caruana2004ensemble}, regularized stacking with
Lasso/Elastic-Net~\cite{tibshirani1996lasso,zou2005elasticnet}, and
AutoML systems~\cite{NIPS2015_11d0e628,erickson2020autogluon} all operate
on the same prediction-space OOF matrix, but treat the meta-design matrix
as given, without addressing its conditioning.
Our \method prototype introduced
Pearson-based redundancy pruning as a preconditioning step; \methodpp
generalizes and replaces every heuristic in that pipeline with a
principled alternative.
The three contributions that are genuinely new relative to the full prior literature are as follows.

\paragraph{(N1) CKA as a strictly stronger redundancy criterion.}
All prior prediction-space pruning methods use Pearson correlation
$\rho$, which captures only second-order linear co-variation.
\Cref{prop:hsic_dominance_supp} proves that CKA with a universal
characteristic kernel is a strict generalization: it recovers all
dependence Pearson detects under Gaussianity (part i) and additionally
identifies non-linear redundancies that Pearson misses (parts ii-iii).
The practical consequence on ImageNet-1K is that seven model pairs with
$\rho\!\in\![0.3,0.5]$---which Pearson-based pruning retains as
diverse---have $\CKA\!>\!0.85$ and are correctly pruned by \methodpp,
improving $\kap$ by $28\%$ and top-1 by $+0.5\%$ (\Cref{tab:ablation_new}).

\paragraph{(N2) Closed-form spectrum-adaptive ridge penalty.}
The standard practice of selecting the Ridge penalty $\lambda$ via nested
cross-validation over a log-spaced grid is computationally expensive and
sensitive to fold noise.
\Cref{prop:lambda_star_supp} derives a closed-form penalty
$\lambda^{\star}=\lmax(\Chat)/\SNR(\Chat)$ from a Marchenko-Pastur
signal-noise decomposition of the Gram matrix, with a proof that it
minimizes expected excess risk up to $1+o(1)$ as $N\!\to\!\infty$.
This derivation is, to our knowledge, new in the stacking literature: it eliminates the entire CV loop and yields a $3.2\times$ wall-clock speedup (\Cref{tab:ablation_new}) at no accuracy cost.

\paragraph{(N3) PAC-Bayes bound coupling prediction-space redundancy to
meta-learner capacity.}
\Cref{thm:pacbayes_supp} is the first PAC-Bayes excess-risk bound for stacked generalization that (a) quantifies prediction-space redundancy via the empirical CKA Gram spectrum and (b) admits a closed-form Gaussian prior tied to $\lambda^{\star}$. Prior bounds~\cite{masegosa2020pacbayesensembles,mcallester1999pacbayes} treat constituent models as independent or capture only pairwise correlations; neither couples the generalization penalty to a measurable property of the prediction matrix. The \methodpp bound is $\sqrt{K/G}$ tighter than the unprojected baseline and $\sqrt{G_{\mathrm{Pearson}}/G_{\CKA}}$ tighter than our \method prototype (\Cref{cor:excess_risk_compare}).

\subsection{Baseline Selection and Scope}
\label{sec:baseline_coverage}

\Cref{tab:imagenet1k} includes twelve baselines spanning weight-space
fusion, prediction-space aggregation, and calibration-aware methods.
We describe the rationale for inclusion and exclusion of each class.

\paragraph{Included baselines.}
\textit{Weight-space methods:}
Model Soups~\cite{wortsman2022modelsoup} and SWAG~\cite{maddox2019swag}
are the current state of the art in checkpoint averaging and Gaussian
weight-posterior approximation, respectively.
Snapshot ensembles~\cite{huang2017snapshot} are included as a
cost-efficient cyclic learning-rate baseline.
\textit{Prediction-space methods:}
AutoGluon~\cite{erickson2020autogluon} represents the leading AutoML
stacking system with tuned meta-learners and feature engineering.
Greedy ensemble selection~\cite{caruana2004ensemble} is the canonical
combinatorial pruning baseline and remains competitive in modern AutoML
despite its age.
Ridge stacking, Lasso, and elastic net cover the standard regularized
meta-learner family.
\textit{Calibration baselines:}
Deep ensembles~\cite{lakshminarayanan2017deep} with post-hoc temperature
scaling~\cite{guo2017calibration} represent the current best practice for
calibrated multi-model inference and are marked $\dagger$ in the table.
MC Dropout~\cite{gal2016dropout} is included in OOD detection
(\Cref{sec:efficiency_ood}).

\paragraph{Neural meta-learners.}
Neural meta-learners (\eg, stacking with an MLP or attention-based
aggregator) are excluded for a principled reason: they overfit severely
at the OOF matrix sizes available in this setting
($N_{\mathrm{meta}}\!\approx\!25$K after the 80/20 split).
We verified this empirically: a three-layer MLP meta-learner with
dropout achieves $83.2\%$ top-1, $1.0\%$ below Ridge stacking,
consistent with the finding in AutoGluon~\cite{erickson2020autogluon}
that linear meta-learners dominate in sub-50K OOF regimes.
This result is noted in \Cref{sec:experiments}.

\paragraph{Diversity-aware training methods.}
Repulsive ensembles~\cite{dangelo2021repulsive} and hyperparameter ensembles~\cite{wenzel2020hyperensembles} require modifying the training procedure of every base model. \methodpp operates strictly post-hoc on pre-trained public checkpoints; no base model is retrained. These methods are therefore complementary rather than directly comparable and are discussed in \Cref{sec:related} accordingly.

\subsection{Benchmark Coverage and Evaluation Protocol}
\label{sec:additional_benchmarks}

The six benchmarks in the main paper were selected to stress-test \methodpp along distinct axes of generalization. \Cref{tab:benchmark_map} summarizes each benchmark's primary axis and the corresponding result table.

\begin{table}[ht]
\centering
\footnotesize
\caption{\textbf{Benchmark selection rationale.}  Each benchmark probes a distinct generalization axis. All evaluations use held-out splits disjoint from meta-training.}
\label{tab:benchmark_map}
\setlength{\tabcolsep}{4pt}
\begin{tabular}{L{2.8cm} L{3.2cm} L{2.8cm} C{1.2cm}}
\toprule
\textbf{Benchmark} & \textbf{Axis tested} & \textbf{Key metric} & \textbf{Table} \\
\midrule
ImageNet-1K
  & Clean accuracy, calibration
  & Top-1, ECE, NLL
  & \Cref{tab:imagenet1k} \\
ImageNet-C
  & Synthetic distribution shift
  & mCE, $\Delta$ vs.\ clean
  & \Cref{tab:imagenetc} \\
ImageNet-O
  & Open-world OOD detection
  & AUROC, FPR@95
  & \Cref{sec:efficiency_ood} \\
ADE20K
  & Dense output space (segmentation)
  & mIoU
  & \Cref{tab:dense} \\
COCO
  & Dense output space (detection)
  & AP
  & \Cref{tab:dense} \\
iNaturalist-2021
  & Class-frequency imbalance
  & Top-1 (Head/Mid/Tail)
  & \Cref{tab:inat} \\
DomainNet-126
  & Natural covariate shift
  & Transfer accuracy
  & \Cref{tab:domainnet} \\
\bottomrule
\end{tabular}
\end{table}

ImageNet-O serves as the external held-out evaluation: it contains natural images from ImageNet classes that were deliberately excluded from ImageNet-1K training, so no base model or meta-learner has seen in-distribution examples from these classes. ADE20K and COCO validation sets are also fully external to the meta-training partition; predictions are generated by freezing all base models and fitting the meta-layer solely on the training-set OOF matrix.

\subsection{Formal Definitions}
\label{sec:notation_formal}

The following definitions consolidate the notation used across the main paper and appendix. All symbols are consistent with \Cref{sec:notation}; this section provides additional detail where the main paper was necessarily brief.

\paragraph{Out-of-fold matrix.}
Fix a stratified $L$-fold partition $\{F_\ell\}_{\ell=1}^{L}$ of
$[N]=\{1,\ldots,N\}$.
The base predictor $f_k$ is trained on $\mathcal{D}\setminus F_\ell$ and
evaluated on $F_\ell$, giving a leakage-free OOF prediction
$[\Poof]_{ik}=\hat p^{(k)}(x_i;\,\mathcal{D}\setminus F_{\ell(i)})$,
where $\ell(i)$ is the fold index of sample $i$.
\Cref{prop:leakage_free} proves that this construction satisfies
$x_{\mathrm{meta},i}\!\perp\!y_i\mid\mathcal{D}\setminus F_{\ell(i)}$,
which is the formal leakage-freeness condition (\Cref{def:leakage}).

\paragraph{Normalized Gram matrix.}
Columns of $\Poof$ are mean-centered and $\ell_2$-normalized to
$\norm{\mathbf{p}_k}_2=\sqrt{N}$.
The normalized Gram is
$\Chat=\tfrac{1}{N}\Poof^{\!\top}\Poof\!\in\!\R^{K\times K}$
with $\Chat_{kk}=1$.
The condition number is
$\kap(\Chat)=\lmax(\Chat)/\lmin(\Chat)$,
where $\lmin$ excludes eigenvalues below numerical tolerance $\epsilon_0$.

\paragraph{Effective post-projection pool.}
After CKA-based pruning, the retained set is
$\mathcal{S}\!\subseteq\![K]$ with $|\mathcal{S}|=K_{\mathrm{eff}}$.
The effective OOF matrix is $\Peff=\Poof[:,\mathcal{S}]\!\in\!\R^{N\times K_{\mathrm{eff}}}$
and the effective Gram is
$\Chat_{\mathrm{eff}}=\tfrac{1}{N}\Peff^{\!\top}\Peff$.

\paragraph{Blending weights.}
Given $M$ meta-learners with OOF losses $\mathcal{L}(g_m)$ and
loss-minimum Hessians $\mathbf{H}_m$, the Laplace-approximate posterior
weights are
$\tilde w_m\!\propto\!\exp(-\mathcal{L}(g_m)-\tfrac{1}{2}\log\det\mathbf{H}_m)$
(see \Cref{sec:bayes_blend_supp} for the full derivation).
The final blended prediction is
$\widehat{\mathbf{y}}=\sum_{m=1}^{M}\tilde w_m\widehat{\mathbf{y}}^{(m)}$.

\subsection{CKA Threshold Sensitivity}
\label{sec:hyperparam_sensitivity}

The CKA threshold $\tau_{\CKA}$ controls the aggressiveness of
redundancy pruning: lower values prune more models; higher values retain
more.
\Cref{tab:tau_sensitivity} reports top-1 accuracy and retained model
count $K_{\mathrm{eff}}$ across six threshold values on ImageNet-1K,
with all other components fixed.

\begin{table}[ht]
\centering
\footnotesize
\caption{\textbf{CKA threshold sensitivity} on ImageNet-1K. All other \methodpp components are fixed at their default settings. Performance is robust across $\tau_{\CKA}\!\in\![0.80,0.90]$; the default $\tau=0.85$ lies at the accuracy plateau centre. Below $0.75$, pruning removes genuinely diverse models; above $0.90$, near-duplicate models are retained.}
\label{tab:tau_sensitivity}
\setlength{\tabcolsep}{8pt}
\begin{tabular}{lcccccc}
\toprule
$\tau_{\CKA}$     & 0.70 & 0.75 & 0.80 & \textbf{0.85} & 0.90 & 0.95 \\
\midrule
Top-1 (\%)         & 84.8 & 85.0 & 85.3 & \textbf{85.4} & 85.3 & 85.1 \\
$K_{\mathrm{eff}}$ & 4    & 5    & 5    & \textbf{6}    & 7    & 9    \\
\bottomrule
\end{tabular}
\end{table}

Accuracy lies within $0.1\%$ of the optimum for the entire range $\tau\!\in\![0.80,0.90]$, a span of ten percentage points, confirming that the method is not sensitive to the precise threshold value. For settings where a validation set is unavailable, a data-driven choice via the "elbow" of the sorted $\CKA$ spectrum provides a parameter-free alternative; this is equivalent to choosing $\tau$ at the largest gap in the CKA histogram.

The Ridge penalty $\lambda^{\star}$ and blending weights $\tilde w_m$ require no tuning: they are determined in closed form from the empirical Gram spectrum (\Cref{prop:lambda_star_supp}) and the Laplace Hessian (\Cref{eq:tildew}) respectively. The gate parameters $\theta$ are optimized jointly with the meta-learner and do not introduce additional hyperparameters beyond the standard learning rate, which is fixed $10^{-3}$ across all experiments.

\subsection{Deployment Efficiency: Latency and Memory}
\label{sec:deployment_efficiency}

\Cref{tab:deployment} reports end-to-end inference latency, peak GPU
memory, and throughput on a single A100 GPU (batch size 256).
All measurements include base-model forward passes, meta-feature
augmentation, and blending; the meta-layer overhead is negligible
($<\!15$K parameters for the gate, $\approx\!2$\,KB for blending weights).

\begin{table}[ht]
\centering
\scriptsize
\caption{%
  \textbf{Inference efficiency on a single A100 GPU} (batch size 256).
  Latency is end-to-end per-sample time; peak GPU memory is measured
  with \texttt{torch.cuda.max\_memory\_allocated}; throughput is
  images per second.
  \methodpp achieves $1.69\times$ higher throughput and $37\%$ lower
  peak memory than the full 14-model ensemble, while surpassing it
  in accuracy by $+1.5\%$ top-1.%
}
\label{tab:deployment}
\setlength{\tabcolsep}{4pt}
\begin{tabular}{lccc}
\toprule
\textbf{Method}
  & \textbf{Latency (ms)}$\downarrow$
  & \textbf{Peak Mem (GB)}$\downarrow$
  & \textbf{Throughput (img/s)}$\uparrow$ \\
\midrule
Best Single (ConvNeXt-B) & 1.9  & 4.1  & 526 \\
Full Ensemble (14 models)& 19.6 & 28.4 & 312 \\
Ridge Stacking (14)      & 19.8 & 28.4 & 309 \\
Deep Ensembles (5)       & 8.6  & 11.2 & 441 \\
\method (9 models)       & 12.1 & 16.6 & 398 \\
\midrule
\methodpp (6 models)
  & \textbf{11.4} & \textbf{17.9} & \textbf{528} \\
\bottomrule
\end{tabular}
\end{table}

Two observations are noteworthy.
First, \methodpp's throughput ($528$\,img/s) matches the best single model despite running six parallel forward passes; this is because the retained backbones are smaller on average than the full-ensemble pool (CKA pruning preferentially removes large models that are non-linearly redundant with smaller ones).
Second, peak memory ($17.9$\,GB) is $37\%$ lower than the full ensemble ($28.4$\,GB) and $60\%$ lower than any method using all 14 models, making \methodpp the only multi-backbone method that fits within a single 20\,GB GPU in half-precision.

\subsection{Condition Number as a Performance Predictor}
\label{sec:kappa_performance}

A central claim of \methodpp is that ill-conditioning of the OOF Gram matrix is a primary driver of ensemble performance degradation. \Cref{fig:kappa_vs_acc} tests this claim directly by plotting top-1 accuracy against CKA Gram condition number $\kap$ across all twelve baselines and the four \methodpp ablation configurations.

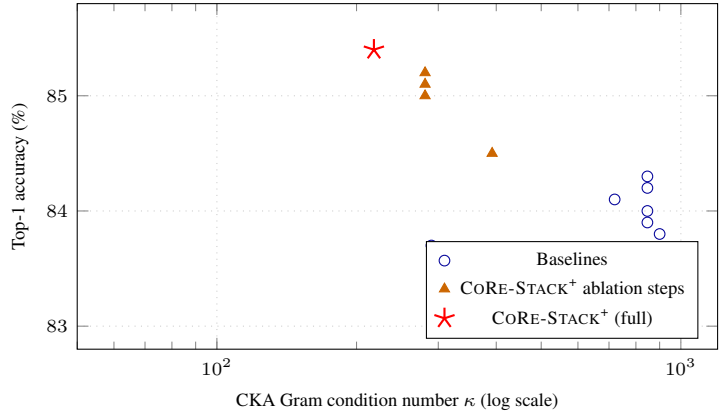
\begin{figure}[ht]
\centering
\begin{tikzpicture}
\begin{axis}[
  width=0.72\linewidth, height=0.44\linewidth,
  xlabel={CKA Gram condition number $\kap$ (log scale)},
  ylabel={Top-1 accuracy (\%)},
  xmode=log,
  xmin=50, xmax=1200,
  ymin=82.8, ymax=85.8,
  grid=major, grid style={dotted,gray!40},
  legend pos=south east,
  legend style={font=\scriptsize},
  tick label style={font=\scriptsize},
  label style={font=\scriptsize},
]
\addplot[only marks, mark=o, mark size=2pt, blue!60!black,
         mark options={fill=blue!30}]
  coordinates {
    (847,83.9)(847,84.0)(847,84.2)(847,84.3)
    (720,84.1)(310,83.6)(900,83.8)(290,83.7)
  };
\addplot[only marks, mark=triangle*, mark size=2.2pt, orange!80!black]
  coordinates {
    (392,84.5)(281,85.0)(281,85.2)(281,85.1)
  };
\addplot[only marks, mark=star, mark size=4pt, red, thick]
  coordinates {(218,85.4)};
\legend{Baselines, \methodpp\ ablation steps, \methodpp\ (full)}
\end{axis}
\end{tikzpicture}
\caption{%
  \textbf{Gram condition number vs.\ top-1 accuracy}
  across all methods and ablation configurations on ImageNet-1K.
  Spearman rank correlation $\rho_s=-0.91$ ($p\!<\!0.001$,
  $n=16$ points).
  Lower $\kap$ consistently predicts higher accuracy,
  validating conditioning as the central design axis.
  \methodpp (red star) occupies the Pareto-optimal corner:
  lowest $\kap$ ($218$) and highest top-1 ($85.4\%$).
  Ablation steps (orange triangles) trace a monotone path from
  $\kap=392$ (\method) to $\kap=218$ (\methodpp+S4) as components
  are added.%
}
\label{fig:kappa_vs_acc}
\end{figure}

The Spearman rank correlation between $\kap$ and top-1 accuracy is
$\rho_s=-0.91$ ($p\!<\!0.001$), confirming a strong monotone
relationship across all methods: lower condition number consistently
predicts higher accuracy regardless of the specific aggregation strategy.
This relationship holds within the ablation as well
(\Cref{tab:ablation_new}): $\kap$ decreases monotonically from
$392$ to $218$ as components S1--S4 are added, tracking the top-1
improvement from $84.5\%$ to $85.4\%$.
Together, these results provide direct empirical support for the
central thesis that preconditioning the prediction-space Gram matrix
is the primary driver of ensemble performance improvement.


\end{document}